\documentclass[english,a4paper,11pt]{article}
\PassOptionsToPackage{pdfpagelabels=false}{hyperref} 
\usepackage{fullpage}
\usepackage{amsxtra, amsfonts, amssymb, amstext}
\usepackage{amsthm}
\usepackage{booktabs}
\usepackage{longtable}
\usepackage{nicefrac}
\usepackage{xspace}
\usepackage[noadjust]{cite}
\usepackage{url}
\usepackage{graphics,color}
\usepackage[colorlinks]{hyperref}
\definecolor{linkblue}{rgb}{0.1,0.1,0.8}
\hypersetup{colorlinks=true,linkcolor=linkblue,filecolor=linkblue,urlcolor=linkblue,citecolor=linkblue}
\usepackage[algo2e,ruled,vlined,linesnumbered]{algorithm2e}
\newcommand{\assign}{\leftarrow}
\usepackage{graphicx}
\usepackage{wrapfig}
\usepackage{caption}
\usepackage{subcaption}
\newtheorem{theorem}{Theorem}
\newtheorem{lemma}[theorem]{Lemma}

\renewcommand{\labelenumi}{(\alph{enumi})}
\renewcommand\theenumi\labelenumi

\allowdisplaybreaks[1]
\newcommand{\N}{\mathbb{N}}
\newcommand{\R}{\mathbb{R}}
\newcommand{\eps}{\varepsilon}
\newcommand{\pmin}{\rho_{\min}}
\newcommand{\pmax}{\rho_{\max}}
\DeclareMathOperator{\mut}{flip}

\DeclareMathOperator{\Geom}{Geom}

\DeclareMathOperator{\rand}{rand}

\DeclareMathOperator{\Bin}{Bin}
\newcommand{\onemax}{\textsc{OneMax}\xspace}

\newcommand{\leadingones}{\textsc{LeadingOnes}\xspace}
\newcommand{\LO}{\textsc{Lo}\xspace}
\newcommand{\oea}{${(1 + 1)}$~EA\xspace}
\newcommand{\ga}{${(1 +(\lambda,\lambda))}$~GA\xspace}
\newcommand{\rls}{\textsc{RLS}\xspace}

\newcommand{\oeares}{${(1 + 1)}$~EA$_{>0}$\xspace}
\newcommand{\oeaopt}{${(1 + 1)}$~EA$_{\text{opt}}$\xspace}
\newcommand{\oearesopt}{${(1 + 1)}$~EA$_{>0,\text{opt}}$\xspace}
\newcommand{\presopt}{p_{>0,\text{opt}}}

\title{Self-Adjusting Mutation Rates\\ with Provably Optimal Success Rules\footnote{An extended abstract announcing the results presented in this work has been communicated at GECCO'19~\cite{DoerrDL19}.}}
\author{Benjamin Doerr$^1$, Carola Doerr$^2$, Johannes Lengler$^3$}
\date{{\footnotesize{
$^1$\'Ecole Polytechnique, CNRS, LIX - UMR 7161, 
91120 Palaiseau, France\\
$^2$Sorbonne Universit\'e, CNRS, LIP6, 75005 Paris, France\\
$^3$Department of Computer Science, ETH Z{\"u}rich, Z{\"u}rich, Switzerland}}\\[1.5ex]
\today
}
\begin{document}
\maketitle 

{\sloppy
\begin{abstract}
The one-fifth success rule is one of the best-known and most widely accepted techniques to control the parameters of evolutionary algorithms. While it is often applied in the literal sense, a common interpretation sees the one-fifth success rule as a family of success-based updated rules that are determined by an update strength $F$ and a success rate. We analyze in this work how the performance of the (1+1) Evolutionary Algorithm on LeadingOnes depends on these two hyper-parameters. Our main result shows that the best performance is obtained for small update strengths $F=1+o(1)$ and success rate $1/e$. We also prove that the running time obtained by this parameter setting is, apart from lower order terms, the same that is achieved with the best fitness-dependent mutation rate. We show similar results for the resampling variant of the (1+1) Evolutionary Algorithm, which enforces to flip at least one bit per iteration. 
\end{abstract}


\section{Introduction}
\label{sec:intro}

One of the key challenges in applying evolutionary algorithms (EAs) in practice is in choosing suitable values for the population sizes, the mutation rates, the crossover probabilities, the selective pressure, and possibly other parameters that determine the exact structure of the heuristic. What complicates the situation is that the optimal values of these parameters may change during the optimization process, so that an ideal parameter setting requires to find not only good initial values, but also suitable update rules that adjust the parameters during the run. \emph{Parameter control}~\cite{EibenHM99} is the umbrella term under which such non-static parameter settings are studied. 

Parameter control is indispensable in continuous optimization, where the step size needs to be adjusted in order to obtain good convergence to solutions of high quality. In this context, non-static parameter choices are therefore standard since the early seventies. In discrete optimization, however, parameter control has received much less attention, as commented in the recent surveys~\cite{KarafotiasHE15,AletiM16}. This situation has changed substantially in the last decade, both thanks to considerable advances in reinforcement learning, which could be successfully leveraged to control algorithmic parameters~\cite{DaCostaGECCO08,FialhoCSS10,KarafotiasEH14}, but also thanks to a number of theoretical results rigorously quantifying the advantages of dynamic parameter settings over static ones, see~\cite{DoerrD20chapter} for a summary of known results.

In continuous optimization, one of the early and still widely used parameter update rule is the \emph{one-fifth success rule}, which was independently designed in~\cite{Rechenberg,SchumerS68,Devroye72}. The one-fifth success rule is derived from the idea that it is desirable to maintain a success rate of around $20\%$, measured by the frequency of offspring having at least the same fitness than the current-best individual. Theoretical justification for this rule was given by Rechenberg, who showed that such a success rate is optimal for controlling the step size of the (1+1) Evolution Strategy (ES) optimizing the sphere function~\cite{Rechenberg}. Based on this finding, several parameter update rules were designed that decrease the step size when the observed success rate is smaller than this target rate, and which increase it for success rates larger than 20\%. 

An interpretation of the one-fifth success rule which is suitable also for parameter control in discrete domains was provided in~\cite{KernMHBOK04}. Kern \emph{et al.} propose to decrease the step size $\sigma$ to $\sigma/F$ after each successful iteration, and to increase it to $\sigma F^{1/4}$ otherwise. 
They propose to consider an iteration successful if the offspring $y$ created in this iteration is at least as good as its parent $x$, i.e., if $f(y) \le f(x)$ in the context of minimizing the function $f$. With this rule, the step size remains constant when one out of five iterations is successful, since in this case after the fifth iteration $\sigma$ has been replaced by $\sigma(F^{1/4})^4/F=\sigma$. This version of the one-fifth success rule, typically using constant update strengths $F>1$, was shown to work efficiently, e.g., in~\cite{Auger09}. In~\cite{DoerrD18ga} it was proven to yield asymptotically optimal linear expected optimization time when applied to the \ga optimizing \onemax. It is also shown in~\cite{DoerrD18ga} that this linear expected running time is optimal up to constant factors, and that no static parameter choice can achieve this efficiency (i.e., all static variants of the $(1+(\lambda,\lambda))$~Genetic Algorithm (GA) require super-linear expected optimization times).

Other success-based multiplicative update rules had previously been studied in the theory of evolutionary algorithms (EAs). For example, L\"assig and Sudholt~\cite{LassigS11} showed that, for the four classic benchmark problems \onemax, \leadingones, \textsc{Jump}, and unimodal functions, the expected number of generations needed to find an optimal solution can be significantly reduced when multiplying the offspring population size $\lambda$ by two after every unsuccessful iteration of the $(1+\lambda)$~EA and when reducing $\lambda$ to $\lambda/2$ otherwise. Similar rules, which also take into account the number of improving offspring, were empirically shown to be efficient in~\cite{JansenJW05} (for $(1+\lambda)$~EA operating on discrete problems) and in~\cite{HansenGO95} for the $(1,\lambda)$~ES optimizing the continuous hyper-plane and hyper-sphere problems. Recently, Doerr and Wagner~\cite{DoerrW18} showed  that success-based multiplicative updates are very efficient for controlling the mutation rate of the \oeares, the resampling \oea variant proposed in~\cite{CarvalhoD17} which enforces to flip at least one bit per each iteration. More precisely, they analyze the average optimization times of the $(1+1)$~EA$_{>0}(A,b)$ algorithm, which increases the mutation rate $p$ by a factor of $A>1$ when the offspring $y$ satisfies $f(y) \ge f(x)$ (i.e., when it replaces its parent $x$) and which decreases $p$ to $bp$, $0<b<1$ otherwise. Their experimental results show that this algorithm has a very good empirical performance on \onemax and \leadingones for broad ranges of update strengths $A$ and $b$. 

\subsection{Our Results}

In this work, we complement the empirical study~\cite{DoerrW18} and rigorously prove that, for suitably chosen hyper-parameters $A$ and $b$, the \oea using this multiplicative update scheme attains the asymptotically optimal expected optimization time across all fitness-dependent mutation rates on the \leadingones function $\LO:\{0,1\}^n \to [0..n]=\{0\}\cup \N_{\le n}, x \mapsto \max\{i \in [0..n] \mid \forall j \le i: x_j=1\}$, where in this work we refer to an optimization time as ``\emph{asymptotically optimal}'' when it is optimal up to lower order terms among all dynamic choices of the mutation rate, and we use optimization time and running time interchangeably.  
For the \oeares we also rigorously prove a bound on the expected optimization time on \leadingones, which we show by numerical evaluations to coincide almost perfectly with the performance achieved by the  \oeares with optimal fitness-dependent mutation rates. 

Following the above-described suggestion made in~\cite{KernMHBOK04}, and adapting to the common notation, we formulate our theoretical results using the parametrization $A=F^s$ and $b=1/F$, where $F>1$ is referred to as the \emph{update strength} and $s$ is referred to as the \emph{success ratio}. As seen above, a success \emph{ratio} of $4$ corresponds to a one-fifth success \emph{rule}. 

We prove that for the \oea the optimal success ratio is $e-1$ (i.e., a $1/e$ success rule). More precisely, we show that the expected running time of the self-adjusting \oea with constant success ratio $s>0$ and small update strength $F=1+o(1)$ on \leadingones is at most $\frac{s+1}{4\ln(s+1)}n^2 + o(n^2)$. 
For $F=1+o(1)$ and $s=e-1$, the expected running time of the self-adjusting \oea is hence $\approx 0.68 n^2 +o(n^2)$, which is asymptotically equivalent to the one of the \oea with optimal fitness-dependent mutation rate~\cite{BottcherDN10}.\footnote{In the preliminary version~\cite{DoerrDL19} of this work, we made the stronger claim that our algorithm is asymptotically optimal among all dynamic choices of the mutation rate of the \oea. While we still believe this to be true, we note that this claim is more substantial than we originally thought. The main difficulty is that it is less obvious than thought whether the result of~\cite{BottcherDN10} extends to all dynamic choices of the mutation rate, in other words, that no asymptotically non-negligible performance gains can be made from letting the mutation rate not only depend on the fitness of the parent, but on the whole history of the process. We strongly believe that such a statement can be shown with the theory of Markov decision processes. Since this would be a deep mathematical analysis not focused on the center of this work (the analysis of multiplicative update rules), we prefer to not conduct this proof and rather reduce our original optimality claim. We are thankful for a comment of an anonymous reviewer that led to the discovery of this gap in our previous optimality statement.}

A key ingredient in our proof is a lemma proving that the mutation rate used by the \oea with self-adjusting mutation rates is, at all times during the optimization process, very close to the \emph{target mutation rate} $\rho^*(\LO(x),s) \approx \ln(s+1)/\LO(x)$, which we define as the unique mutation rate that leads to success probability $1/(s+1)$. 

We also extend our findings to the \oeares considered in~\cite{DoerrW18}. This resampling \oea variant is technically more challenging to analyze, since the probabilities of the conditional standard bit mutation operator (which enforces to flip at least one bit)  are more complex to handle, but also because the concept of target mutation rates ceases to exist for fitness levels $\ell \geq \tfrac{s}{s+1} n$: it is impossible to achieve success rates of $1/(s+1)$ or higher for such values of $\ell$ without accepting duplicates as offspring. In this regime, the mutation rate approaches zero, and the \oeares resembles Randomized Local Search (RLS), which flips in each iteration exactly one randomly chosen bit. This behavior is desirable since the optimal number of bits to flip in solutions $x$ with $\LO(x)\ge n/2$ is indeed equal to one.
 
In contrast to the unconditional \oea, our bound for the expected running time of the self-adjusting \oeares does not seem to have a straightforward closed-form expression. A numerical evaluation for dimensions up to $n=10\,000$ shows that the best running time is achieved for success ratio $s \approx 1.285$. With this choice (and using again $F=1+o(1)$), the performance of the self-adjusting  \oeares is almost indistinguishable from that of the \oearesopt, i.e., the \oeares using the optimal fitness-dependent mutation rate. Both algorithms achieve an expected running time for $n=10\,000$ which is around $0.404n^2$. 

For both algorithms, the self-adjusting \oea and the self-adjusting \oeares, we do not only bound the expected optimization time but we also prove \emph{stochastic domination bounds}, which provide much more information about the running time~\cite{Doerr19tcs}. We only show upper bounds in this work, but we strongly believe that our bounds are tight, since for the \oea we obtain an asymptotically optimal running time, and for the self-adjusting \oeares the numerical bounds are almost indistinguishable from those of the \oearesopt.

Finally, we take into account suggestions made in~\cite{CarvalhoD17,BuzdalovDDV20} and briefly comment on the fixed-target optimization times, i.e., the number of evaluations needed by the different algorithms to sample for the first time a solution that satisfies a minimal quality requirement. Since our proofs are based on fitness-level arguments, such statements can be obtained rather straightforwardly. 
 
\subsection{Related Work}

In~\cite{LissovoiOW17,LissovoiOW19}, variants of RLS flipping a deterministic, but dynamic number of bits were analyzed on \leadingones. These schemes, which take inspiration from the literature on hyper-heuristics, differ from our rather simple multiplicative parameter updates in that they require additional book-keeping and that they consider a small discrete set of possible parameter values only. The approach is taken further in~\cite{DoerrLOW18}, where a hyper-parameter from~\cite{LissovoiOW17} is dynamically adapted by a rule inspired by the one-fifth success rule. This algorithm is analyzed in the case that the algorithm may choose between flipping either one or two bits, and it is proven that it achieves an asymptotically optimal expected running time on \leadingones among all algorithms which are restricted to these two choices. However, the resulting algorithm is considerably more complicated than our simple scheme. Moreover, while the algorithm in~\cite{DoerrLOW18} is asymptotically optimal among all algorithms which---based only on the fitness of the current best solution---use either one-bit flips or two-bit flips and while the ones from~\cite{LissovoiOW19} are asymptotically optimal among all fitness-dependent algorithms choosing between 1 and a constant number $k$ of bits to flip, we prove asymptotic optimality among all algorithms with fitness-dependent standard bit mutation for which the number of flipped bits is random and can attain any value between 0 (1 in the case of the resampling \oea) and~$n$. 

\section{The Self-Adjusting (1+1) EA}
\label{sec:oeadef}

We study the optimization time of the \oea with self-adjusting mutation rates, which is summarized in Algorithm~\ref{alg:adaoea}. The self-adjusting \oea starts the optimization process with an initial mutation rate $\rho=\rho_{0}$ and a random initial solution $x \in \{0,1\}^n$. In every iteration, one new solution candidate $y\in \{0,1\}^n$ is created from the current-best solution $x$ through standard bit mutation with mutation rate $\rho$, i.e., the \emph{offspring} $y$ is created from $x$ by flipping each bit, independently of all other decisions, with probability $\rho$. If $y$ is at least as good as its parent $x$, i.e., if $f(y)\ge f(x)$, $x$ is replaced by $y$ and the mutation rate $\rho$ is increased to $\min\{F^{s}\rho,\rho_{\max}\}$, where $F>1$ and $s>0$ are two constants that remain fixed during the entire execution of the algorithm and $0<\rho_{\max} \le 1$ is an upper bound for the range of admissible mutation rates. If, on the other hand, $y$ is strictly worse than its parent $x$, i.e., if $f(y)<f(x)$, then $y$ is discarded and the mutation rate is decreased to $\max\{ \rho/F,\rho_{\min} \}$, where $0<\rho_{\min}$ is the smallest admissible mutation rate. The algorithm continues until some stopping criterion is met. Since in our theoretical analysis we know the optimal function value $f_{\max}$, we use as stopping criterion that $f(x)=f_{\max}$.

\textbf{Standard Bit Mutation.} Since we will also consider the \oeares, which requires that each offspring $y$ differs from its parent $x$ in at least one bit, we use in lines~\ref{line:k} and~\ref{line:mut} the equivalent description of standard bit mutation, in which we first sample the number $k$ of bits to flip and then apply the mutation operator $\mut_k$, which flips $k$ pairwise different, uniformly chosen bits in $x$.

\textbf{Success Ratio vs. Success Rule.} We recall from the introduction that we call $F$ the \emph{update strength} of the self-adjustment and $s$ the \emph{success ratio}. The success ratio $s=4$ is particularly common in evolutionary computation~\cite{KernMHBOK04,Auger09,DoerrD18ga}. With this choice $s=4$ the parameter update mechanism is well known as the \emph{one-fifth success rule}, in the interpretation suggested in~\cite{KernMHBOK04}: if one out of five iterations is successful, the parameter $\rho$ stays constant (since it will have been updated to $\rho (F^{1/4})^4/F=\rho$). Note that a success \emph{ratio} of $s$ corresponds to a one-$(s+1)$-th success \emph{rule}. We choose to work with success ratios instead of success rates for notational convenience.

\textbf{Hyper-Parameters.} Altogether, the self-adjusting \oea has five \emph{hyper-parameters}: the update strength $F$, the success rate $s$, the initial mutation rate $\rho_{0}$, and the minimal and maximal mutation rates $\rho_{\min}$ and $\rho_{\max}$, respectively. It is not difficult to verify that for update strengths $F=1+\eps$, $\eps = \Omega(1)$, the mutation rate deviates, in at least a constant fraction of all iterations, from the optimal one by at least a constant factor, which results in a constant factor overhead in the running time. We therefore consider $F=1+o(1)$ only. Apart from this, we only require that $\rho_{\min}=o(n^{-1}) \cap \omega(n^{-c})$ for an arbitrary constant $c$ and, for mathematical simplicity, we set $\rho_{\max} = 1$. Note though that for practical applications of the algorithm, we suggest to bound $\rho_{\min} \ge 1/n^2$ and $\rho_{\max}\leq 1/2$. 

With these specifications, we are left with the success ratio $s$. Our main interest is in bounding the running time of the self-adjusting \oea in dependence of this parameter. 
 
Finally, we note that Algorithm~\ref{alg:adaoea} generalized the classic \oea with static mutation rate $\rho$, which we obtain by setting $F=1$ and $\rho_{0}=\rho$. 

 \begin{algorithm2e}[t]%
	Sample $x \in \{0,1\}^{n}$ uniformly at random and compute $f(x)$\;
	Set $\rho=\rho_{0}$\;
	\While{\text{stopping criterion not met}}{
		\label{line:k}Sample $k$ from $\Bin(n,\rho)$\;
		\label{line:mut}$y \assign \mut_{k}(x)$\;
		evaluate $f(y)$\;
		\eIf{$f(y)\geq f(x)$}
			{$x \assign y$ and $\rho \assign \min\{ F^s \rho, \rho_{\max}\}$}
			{$\rho \assign \max\{\rho/F, \rho_{\min}\}$}	
}
\caption{The self-adjusting \oea with update strength $F$, success ratio $s$, initial mutation rate $\rho_{0}$, minimal mutation rate $\rho_{\min}$, and maximal mutation rate $\rho_{\max}$. The formulation assumes maximization of the function $f:\{0,1\}^n \rightarrow \R$ as objective.}
\label{alg:adaoea}
\end{algorithm2e}

\textbf{Improvement vs. Success Probability.} We study in this work the performance of the self-adjusting \oea on the \leadingones function 
$$\LO:\{0,1\}^n \to \R, x \mapsto \max\{ i \in [0..n] \mid \forall j \le i: x_j=1\},$$ 
which counts the number of initial ones in a bit string. By the unbiasedness of the algorithms studied in this work in the sense of~\cite{LehreW12} all our results also hold for perturbed versions of LeadingOnes that are obtained by any transformation of the search space that preserves Hamming distances. 
We build our analysis on results presented in~\cite{BottcherDN10,Doerr19tcs}, which reduce the study of the overall running time to analyzing the time spend on each fitness level. More precisely, for a random solution $x \in \{0,1\}^n$ with $f(x)=:\ell$ we study the time $T_\ell$ that it takes the self-adjusting \oea to reach for the first time a solution $y$ of fitness $f(y)>\ell$. We call the probability to create such a $y$ the \emph{improvement probability} $p_{\text{imp}}(\rho,\ell)$ of mutation rate $\rho$ on level $\ell$. For fixed mutation rate $\rho$, this improvement probability is easily seen to equal $(1-\rho)^{\ell}\rho$, since the first $\ell$ bits should not flip, the $(\ell+1)$-st should, and it does not matter what happens in the \emph{tail} of the string.  

Another important probability is the \emph{success probability} $p_{\text{suc}}(\rho,\ell):=(1-\rho)^{\ell}$ of creating an offspring $y$ that is \emph{at least as good} as $x$, since this is the probability of increasing the mutation rate from $\rho$ to $\min\{F^{s}\rho,\rho_{\max}\}$. 

We note that several other works studying self-adjusting parameter choices assume that the adjustment rule distinguishes whether or not a strict improvement has been found. In the analysis of the self-adjusting \ga in~\cite{DoerrD18ga}, for example, it is assumed that $\lambda \assign \lambda/F$ if and only if $f(y)>f(x)$, while $\lambda$ is updated to $\lambda F^{1/4}$ otherwise. 
While formally analyzing the impact of this choice (which requires a quantification of the expected running time for the self-adjusting \oea using such an alternative update rule) goes beyond the scope of this present work, it is certainly desirable to develop general guidelines which update rule to prefer for which type of problems. The empirical investigations of the two alternatives reported in~\cite{RodionovaABD19} suggests that such guidelines are non-trivial to derive. 

\textbf{Asymptotic analysis:} Our result is an asymptotic running-time analysis, that is, we are interested in the running-time behavior for large problems sizes $n$. More formally, we view the running time $T$ as a function of the problem size $n$ (even though we do not explicitly write $T(n)$) and we aim at statements on its limiting behavior. As usual in the analysis of algorithms, we use the Landau symbols $O(\cdot)$, $\Omega(\cdot)$, $\Theta(\cdot)$, $o(\cdot)$, and $\omega(\cdot)$ to conveniently describe such limits. When using such a notation, we shall always view the expression in the argument as a function of $n$ and we use the notation to describe the behavior for $n$ tending to infinity. We note that already the algorithm parameters $\eps$ and $\pmin$ are functions of $n$ (which is very natural since it just means that we use different parameter values for different problem sizes). Different from $\eps$ and $\pmin$, we take $s$ as a constant (that is, not depending on $n$). Success rates varying with the problem size have been shown useful in~\cite{DoerrLOW18}, but generally it is much more common to have constant success rates and we do not see how non-constant success rates could be advantageous in our setting.

\section{Summary of Useful Tools}
\label{sec:tools}

We shall frequently use the following well-known estimates. Note for \ref{it:estsecondorder} and \ref{it:estthirdorder} that the binomial coefficients $\binom{s}{k} := s\cdot \ldots \cdot (s-k+1)/k!$ are defined for all $s \in \R$ and $k \in \N$.
\begin{lemma}\label{lem:est}
\begin{enumerate}
\item \label{it:est1} For all $x \in \R$, $1+x \le e^x$.
\item \label{it:est2} For all $x < 1$, $e^{x} \le 1 + \frac x{x-1}$.
\item \label{it:est3} For all $x \in [0,1]$, $e^{-x} \le 1 - \frac x2$.
\item \label{it:estbernoulli} For all $x \ge -1$ and $s \ge 1$, $(1+x)^s \ge 1 + sx$.
\item \label{it:estbernoulli2} For all $x \ge -1$ and $0 \le s \le 1$, $(1+x)^s \le 1 + sx$.
\item \label{it:estsecondorder} For all $0 \le x \le 1$ and $s \ge 2$, $(1-x)^s \le 1 - sx +\binom{s}{2}x^2$.
\item \label{it:estthirdorder} For all $0 \le x \le 1$ and $s \ge 3$, $(1-x)^s \ge 1 - sx +\binom{s}{2}x^2- \binom{s}{3}x^3$.
\end{enumerate}
\end{lemma}

\begin{proof}
Part~\ref{it:est1} to \ref{it:estbernoulli} can be found, for example, in~\cite[Lemmas 4.1, 4.2 and 4.8]{Doerr20bookchapter}. Part~\ref{it:estbernoulli2} follows easily from~\ref{it:estbernoulli}. For Part~\ref{it:estsecondorder} and~\ref{it:estthirdorder}, by Taylor's theorem, e.g.~\cite[Theorem 2.5.4]{trench2013introduction}, for any function $f:\R \to \R$ that is $k+1$ times continuously differentiable on an open interval $I$,  and for any $x,a \in I$ there exists $\xi$ between $a$ and $x$ such that
\[
f(x) = f(a) + \left(\sum_{i=1}^{k}\frac{f^{(i)}(a)}{i!}(x-a)^i \right)- \frac{f^{(k+1)}(\xi)}{(k+1)!}(x-a)^{k+1}.
\]
We use this theorem for $f(x) = (1-x)^s$ with $a=0$. The main term of the expansion corresponds to the right hand side of~\ref{it:estsecondorder} and~\ref{it:estthirdorder} for $k=2$ and $k=3$ respectively. For~\ref{it:estsecondorder}, since $f^{(3)}(\xi) = -s(s-1)(s-2)(1-\xi)^{s-3}$, the error term $f^{(k+1)}(\xi)/(k+1)! \cdot(x-a)^{k+1}$ is non-positive for all $\xi \in [0,x]$, and the claim follows. Likewise,~\ref{it:estsecondorder} follows since $f^{(4)}(\xi) = s(s-1)(s-2)(s-3)(1-\xi)^{s-4}$ is non-negative for all $\xi \in [0,x]$.
\end{proof}

Sometimes we need more precise error terms. In this case, we resort to the following asymptotic expansions around zero. That is, we will use the following expansions in the case that $x$ and/or $y$ are close to zero.
\begin{lemma}\label{lem:asym}
Let $c_1,c_2 \in \R$, $c_1 < c_2$, and consider the interval $I = [c_1,c_2] \subseteq \R$. 
Then we have the following asymptotic expansions, where the hidden constants only depend on $c_1$ and $c_2$:
\begin{enumerate}
\item \label{it:asym1} For all $x \in I$, $e^{-x} = 1-x+x^2/2 \pm O(|x|^3)$.
\item \label{it:asym2} If $c_1 > -1$ and $c_2 <1$ then for all $x \in I$, $1/(1+x) = 1-x+x^2 \pm O(|x|^3)$.
\item \label{it:asym3} If $c_1 > -1$ and $c_2 <1$ then for all $x \in I$, $\ln(1+x) = x-x^2/2 \pm O(|x|^3)$. Equivalently, we may write $1+x = e^{x-x^2/2 \pm O(|x|^3)}$.
\item \label{it:asym4} For all $0<y \in I$ and (for concreteness) all $x \in [-1/2,1/2]$,
\[
1-e^{-y}(1+x) = (1-e^{-y})\left(1-x/y+x/2\pm O(|y x|)\right)
\]
\end{enumerate}
Note that the lemma explicitly allows negative values of $x$. To emphasize the exact meaning of the $O$-notation, we spell out exemplarily the meaning of the second statement. It says that for all $c_1,c_2 \in \R$ with $-1 <c_1<c_2<1$ there exists $C >0$ such that for all $x \in [c_1,c_2]$,
\[
1-x+x^2 - C \cdot |x|^3 \leq 1/(1+x) \leq 1-x+x^2 + C \cdot |x|^3
\]
\end{lemma}

\begin{proof}
Parts~\ref{it:asym1},~\ref{it:asym2}, and~\ref{it:asym3} are standard Taylor expansions, see for example~\cite[Examples 2.5.1-2.5.3]{trench2013introduction}. Part~\ref{it:asym4} follows from the other parts by the following calculation.
\begin{align*}
1-e^{-y}(1+x) & = (1-e^{-y})\Big(1-\frac{x}{1-e^{-y}} + x\Big) \stackrel{\ref{it:asym1}}{=} (1-e^{-y})\Big(1-\frac{x}{y(1-y/2\pm O(y^2))}+x\Big) \nonumber \\
& \stackrel{\ref{it:asym3}}{=} (1-e^{-y})\left(1-\tfrac{x}{y}(1+y/2 \pm O(y^2))+x\right) \nonumber \\
& = (1-e^{-y})\left(1-x/y+x/2\pm O(|y x|)\right).
\end{align*}
\end{proof}

The following lemma will be helpful to estimate the success probabilities for the \oeares. 
\begin{lemma}\label{lem:auxmonotone}
For every $0 < b < c$, the function $f(x) = (1-b^x)/(1-c^x)$ is strictly decreasing in $x$ in the range $x \in \R^+$.
\end{lemma}

\begin{proof}
We first argue that for all $a\in \R$, the auxiliary function $g(y) := (e^{-ay}-1)/y$
is strictly increasing in $y \in \R^-$. To check this, we compute the derivative $g'(y) = \frac{e^{-ay}}{y^2}(e^{ay}-ay-1)$, which is strictly positive by Bernoulli's inequality, Lemma~\ref{lem:est}\ref{it:estbernoulli}. Thus $g$ is strictly increasing in $\R^- := \{x\in \R\mid x<0\}$. (It is also strictly increasing in $\R^+$, but since it has a pole in $0$, it is not increasing in all of $\R\setminus\{0\}$.) In particular, setting $a:= x$ and comparing $g(y)$ for $y=\ln b$ and $y=\ln c$ yields $(b^{-x}-1)/\ln b < (c^{-x}-1)/\ln c$, or equivalently (mind that $\ln b$ and $\ln c$ are both negative)
\begin{equation}\label{eq:rhostarhat1}
(1-b^x)c^x\ln c < (1-c^x)b^x\ln b.
\end{equation}
With this preparation, we observe for the derivative of $f$,
\[
f'(x) = \frac{(1-b^x)c^x\ln c - (1-c^x)b^x\ln b}{(1-c^x)^2} \stackrel{\eqref{eq:rhostarhat1}}{<} 0.
\]
This proves the lemma.
 \end{proof}

We also need the following result showing that a random process with negative additive drift in the non-negative numbers cannot often reach states that are mildly far in the positive numbers. In other words, the \emph{occupation probability} of such states is small. Results of a similar flavor have previously been obtained in~\cite{DoerrWY21}, but we do not see how to derive our result easily from that work. 

\begin{lemma}\label{lem:occprob}
  Let $D$ be a discrete random variable 
	satisfying $|D| \le s$ and $E[D] = -\phi$ for some $\phi \le \sqrt 2 \, s$. Let $X_t$ be a random process on $\R$ such that 
  \begin{itemize}
  \item $(X_t)$ starts on a fixed value $x_0 \le s$, that is, we have $\Pr[X_0 = x_0] = 1$,
  \item for all $t$ and for all $r_1, \dots, r_t \in \R$ with $\Pr[\forall i \in [1..t] : X_i = r_i] > 0$ we have
  \begin{itemize}
  \item if $r_t \ge 0$, then conditioned on $X_1=r_1, \ldots, X_t = r_t$, the conditional distribution of $X_{t+1}$ is dominated by $r_t + D$,
  \item if $r_t < 0$, then $X_{t+1} - X_t$ has a discrete distribution with absolute value at most~$s$.
  \end{itemize}
  \end{itemize}
  Then for all $t$ and $U \ge s$, we have \[\Pr[X_t \ge U] \le \exp\left(-\frac{\phi (U-s)}{2s^2}\right) \left(\frac{U-s}{\phi} + 1 + \frac {4s^2}{\phi^2} \right).\] In particular, for $U = 6 \frac {s^2}{\phi} \ln(\frac 1 \phi) + s$, we have $\Pr[X_t \ge U] \le 6\phi s^2 \ln(\frac 1 \phi) + \phi^3 + 4\phi s^2$, an expression tending to zero for $\phi \to 0$.
\end{lemma}

\begin{proof}
  One potential problem in analyzing the process $(X_t)$ is that we have not much information about its behavior when it is below zero. We solve this problem via a sequence of domination arguments.
  
  We think of the process $(X_t)$ as a particle moving on the real line. When the particle leaves the non-negative part, we ignore what it is doing until it reappears at some later time at some position in $[0,s)$, which is determined also by the part of the process which we do not understand. 
  
  To gain an upper bound for $\Pr[X_t \ge U]$, as a first pessimistic simulation of this true process, we may pessimistically assume that an adversary puts the particle on an arbitrary position in $[0,s)$ when it reappears. The adversary could make this position depend on the previous walk of the particle, but clearly there is no gain from this (the adversary should just choose a position which maximizes $\Pr[X_t \ge U]$). 
  
  To overcome the difficulty that we do not know when the particle reappears, we regard the following pessimistic version of the previous process with adversarial reappearances. At each time $t' \in [0..t]$, the adversary adds a particle to the process (in the interval $[0,s)$). At each time step, each previously present particle that is in $[0,\infty)$ performs a move distributed as $D$ (independent for all particles and all times). When a particle enters the negative numbers, it disappears. It is clear that the probability that there is some particle at $U$ or higher at time $t$ is at least $\Pr[X_t \ge U]$.
  
  To avoid having to deal with the disappearance of particles in the previous process, we do not let them disappear, but we let them perform a modified walk when in the negative numbers. Clearly this can only increase the number of particles on each position at each time. In the modified walk also in the negative numbers the particles perform steps distributed according to~$D$. 
  
  In summary, we regard the process in which at each time step $t' \in [0..t]$ a particle $t'$ appears at some prespecified position $x_{t',t'}$. Each previously existing particle $i < t'$ moves in this time step $t'$ to the new position $x_{t',i} = x_{t'-1,i} + D$, independent from the past and independently for all particles. We are interested in the number $Y_t := |\{i \in [0..t] \mid x_{t,i} \ge U\}|$ of particles that are in $[U,\infty)$ at time $t$, since we have $\mathbf{1}_{X_t \ge U} \preceq Y_t$ and hence $\Pr[X_t \ge U] = E[\mathbf{1}_{X_t \ge U}] \le E[Y_t]$.
  
  After this slightly lengthy preparation, we now quickly estimate $E[Y_t]$. By linearity of expectation, $E[Y_t] = \sum_{i = 0}^t \Pr[x_{t,i} \ge U]$. By construction, $x_{t,i} \le s + \sum_{j = 1}^{t-i} D_j$, where the $D_j$ are independent copies of $D$. Writing $D^{t-i} = \sum_{j = 1}^{t-i} D_j$, we compute 
  \begin{align*}
  \Pr[x_{t,i} \ge U] & \le \Pr[D^{t-i} \ge U-s] \\
  & = \Pr[D^{t-i} \ge E[D^{t-i}] + (t-i)\phi + U-s]\\
  & \le \exp\left(-\frac{((t-i)\phi + U-s)^2}{2 (t-i) s^2}\right),
  \end{align*}
  where the last estimate stems from the additive Chernoff bound (e.g., Theorem~10.9 in~\cite{Doerr20bookchapter}). Using the estimates 
  \begin{itemize}
  \item $\Pr[x_{t,i} \ge U] \le \exp(- \phi^2 (t-i) / 2s^2)$ for $t-i \ge \frac{U-s}{\phi}$, 
  \item $\Pr[x_{t,i} \ge U] \le \exp(- \phi (U-s) / 2s^2)$ for $t-i < \frac{U-s}{\phi}$, and
  \item $\sum_{k=0}^\infty \exp(-\phi^2/2s^2)^k = 1 / (1-\exp(-\phi^2/2s^2)) \le 1 / (1 - (1 - \phi^2/4s^2)) = 4s^2/\phi^2$ (using $\phi \le \sqrt 2 \, s$ and Lemma~\ref{lem:est}~\ref{it:est3}), 
  \end{itemize}
  we obtain 
  \begin{align*}
  E[Y_t] &\le \left(\frac{U-s}{\phi}+1\right) \exp\left(-\frac{\phi (U-s)}{2s^2}\right) + \sum_{k = \lceil \frac{U-s}{\phi}\rceil}^{\infty} \exp\left(-\frac{\phi^2 k}{2s^2}\right)\\
  & \le \exp\left(-\frac{\phi (U-s)}{2s^2}\right) \left(\frac{U-s}{\phi} + 1 + \frac {4s^2}{\phi^2} \right).
  \end{align*}
\end{proof}

\section{Analysis of the Self-Adjusting (1+1) EA}
\label{sec:oea}

Theorem~\ref{thm:oea} summarizes the main result of this section. Before providing the formal statement, we introduce a quantity that will play an important role in all our computations, the \emph{target mutation rate} $\rho^*(\ell,s)$. We consider as target mutation rate the value of $p$ which leads to the success probability that is given by the success rule. That is, for each fitness level $1 \le \ell \le n-1$ and each success ratio $s>0$ the target mutation rate  
$\rho^*(\ell,s)$ is the unique value $p \in (0,1)$ that satisfies $p_{\text{suc}}(p)= (1-p)^{\ell}=1/(s+1)$. For $\ell=0$ we set $\rho^*(\ell,s) := 1$. A key argument in the following proofs will be that the mutation rate drifts towards this target rate.

Following the discussion in~\cite{Doerr19tcs} we do not only analyze in Theorem~\ref{thm:oea} the expected running time, but rather show a \emph{stochastic domination} result. To formulate our results, we introduce the shorthand $X \preceq Y$ to express that the random variable $X$ is stochastically dominated by the random variable $Y$, that is, that $\Pr[X \ge r] \le \Pr[Y \ge r]$ for all $r \in \R$. We also recall that a random variable $X$ has a geometric distribution with success rate~$p$, written as $X \sim \Geom(p)$, when $\Pr[X = k] = (1-p)^{k-1} p$ for all $k = 1, 2, \dots$.

\begin{theorem}\label{thm:oea}
  Let $c > 1$ be a constant. Consider a run of the self-adjusting \oea with $F=1+\eps$, $\eps \in \omega(\frac{\log n}{n}) \cap o(1)$, $s > 0$, $\pmin \in o(n^{-1}) \cap \Omega(n^{-c})$, $\pmax = 1$, and arbitrary initial rate $\rho_{0} \in [\pmin,\pmax]$ on the $n$-dimensional \leadingones function. Then the number $T$ of iterations until the optimum is found satisfies
  \begin{align*}
  T & \preceq o(n^2) + \sum_{\ell = 0}^{n-1} X_\ell \Geom\left(\min\{\omega(\tfrac 1n), (1-o(1))(1-\rho^*(\ell,s))^\ell \rho^*(\ell,s)\}\right),
  \end{align*}
  where the $X_\ell$ are uniformly distributed binary random variables and all $X_\ell$ and all geometric random variables are mutually independent. Further, all asymptotic notation solely is with respect to $n$ and can be chosen uniformly for all $\ell$. In particular, 
  \begin{align}
	\label{eq:Toeagen}
  E[T] & \le (1+o(1)) \frac 12 \sum_{\ell=0}^{n-1} \left((1-\rho^*(\ell,s))^\ell \rho^*(\ell,s)\right)^{-1} = (1+o(1))\frac{s+1}{4\ln(s+1)} n^2.
  \end{align}
\end{theorem}

\begin{figure}[t]
\centering
\includegraphics[width=0.5\linewidth]{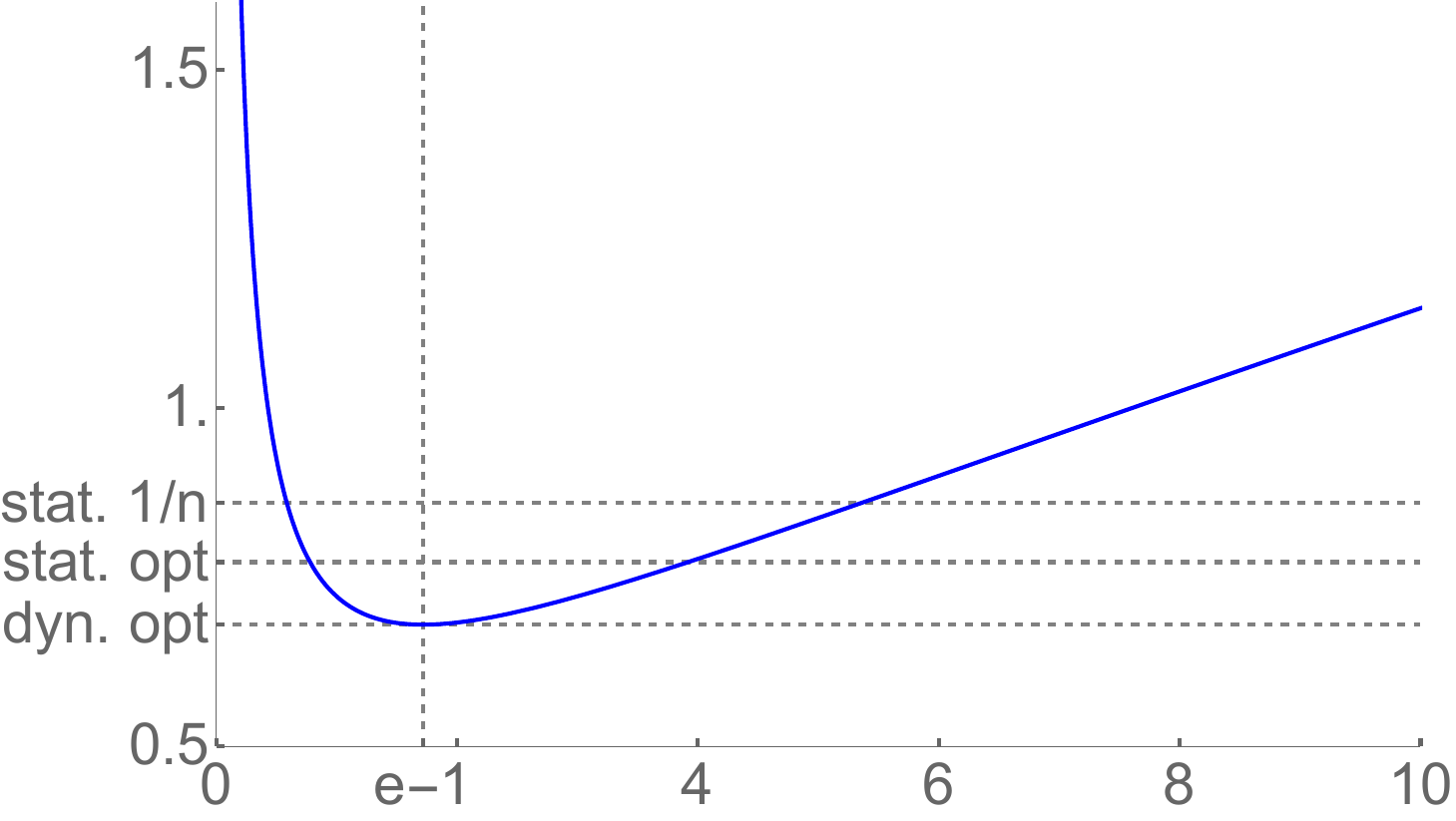}
\caption{Normalized (by the factor $1/n^2$) expected optimization times of the self-adjusting \oea for different success ratios $s$, and assuming $F=1+o(1)$.}
\label{fig:runtime-s}
\end{figure}

\subsection{Numerical Evaluation of the Running-Time Bound in Theorem~\ref{thm:oea}}
\label{sec:numerics-oea}
Figure~\ref{fig:runtime-s} displays the normalized (by the factor $1/n^2$) expected optimization time $\frac{s+1}{4\ln(s+1)}$ for success ratios $0 < s \le 10$. Minimizing this expression for $s$ shows that a success ratio of $s=e-1$ is optimal. With this setting, the self-adjusting \oea yields an expected optimization time of $(1\pm o(1))e n^2/4$, which was shown in~\cite{BottcherDN10} to be optimal for the \oea across all possible fitness-dependent mutation rates. In fact, with this success ratio, it holds that $\rho^*(\ell,s) \approx 1/(\ell+1)$, which is the mutation rate that was shown in~\cite{BottcherDN10} to be optimal when applied to mutating a random solution $x$ satisfying $\LO(x)=\ell$. 

Using Equation~\eqref{eq:Toeagen} we can also compute that for all success ratios $s \in [0.78,3.92]$ 
the expected optimization time of the self-adjusting \oea is better than the $0.77201 n^2+o(n^2)$ one of the best static \oea computed in~\cite{BottcherDN10}, which uses mutation rate $p^*=1.5936.../n$. We furthermore see that the one-fifth success rule (i.e., using $s=4$) performs slightly worse; its expected optimization time is around $0.7767 n^2 +o(n^2)$. Note, however, that we will see in Section~\ref{sec:fixedTarget} (see also Figure~\ref{fig:fixedTarget}), that its fixed-target performance is nevertheless better than the static one with $p^*$ for a large range of sub-optimal target values.

Finally, we note that for success ratios $s \in [0.59,5.35]$ 
the expected optimization time of the self-adjusting \oea is better than the $0.85914...n^2+o(n^2)$ expected running time of the static \oea with default mutation rate $p=1/n$. 

\subsection{Proof Overview}
 
The main proof idea consists in showing that in a run of the self-adjusting \oea we sufficiently often have a mutation rate that is close to the target mutation rate (the unique rate which gives a success probability of $1/(s+1)$). We show this by proving that the self-adjustment leads to a drift of the mutation rate towards the target rate $\rho^*(\ell,s)$. This drift is strong when the rate is far from the target, and we can use a multiplicative drift argument~\cite{DoerrJW12} to show that the rate quickly becomes close to the target rate (Lemma~\ref{lem:approachrho}). Once close, we use our occupation probability lemma (Lemma~\ref{lem:occprob}) based on additive drift to argue that the rate is often at least mildly close to the target rate (Lemma~\ref{lem:rhoisgood}). We need a careful definition of the lower order expressions ``often'', ``close'', and ``mildly close'' to make this work.

From the knowledge that the mutation rate is often at least mildly close to the target rate, we would like to derive that the optimization process is similar to using the target rate in each iteration. This is again not trivial; a main obstacle is that the rate is not chosen independently in each iteration. Consequently, we cannot argue that each iteration on one fitness level has the same, independent probability for finding an improvement (which would give that the waiting time on the level follows a geometric distribution). We overcome this difficulty by splitting the time spent on one fitness level in short independent phases, each consisting of bringing the rate into the desired region and then exploiting that the rate will stay there most of the time (Lemma~\ref{lem:gaintime}). This approach is feasible because of our relatively good bounds for the time needed to reach the desired rate range. The final argument is that the expected running time of the self-adjusting \oea on \leadingones is half the sum of the expected times needed to leave each fitness level. Such a statement has been previously observed for static and fitness-dependent mutation rates~\cite{BottcherDN10,Doerr19tcs}.

Since we are interested in asymptotic results only, we can and shall assume in the remainder that $n$ is sufficiently large.

\subsection{Proof of Theorem~\ref{thm:oea}}
\label{sec:proofoea}

As a first step towards understanding how our EA adjusts the mutation rate, we first determine and estimate the target mutation rate $\rho^*(\ell,s)$ introduced in the beginning of Section~\ref{sec:oea}. We shall use the following result frequently and often without explicit notice.

\begin{lemma}[estimates for the target mutation rate $\rho^*(\ell,s)$]
\label{lem:rhostar}
  Let $\ell \ge 1$ and $\rho^* = \rho^*(\ell,s)$. Then $\rho^* = 1 - (s+1)^{-1/\ell}$ and 
  \[\frac{\ln(s+1)}{\ell} \left(1+\frac{\ln(s+1)}{\ell}\right)^{-1} \le \rho^* \le \frac{\ln(s+1)}{\ell}.\] Consequently, $\rho^* = \Theta(\frac{1}{\ell})$ and $\rho^* \le \rho^*(1,s) < 1$ is bounded away from $1$ by at least a constant. If $\ell = \omega(1)$, then $\rho^* = (1-o(1)) \frac{\ln(s+1)}{\ell}$. 
\end{lemma}

\begin{proof}
  The precise value for $\rho^*$ follows right from the definition of $\rho^*$. Rewriting \[\rho^* = 1 - (s+1)^{-1/\ell} = 1 - \exp(-\tfrac 1 \ell \ln(s+1))\] and using the estimates from Lemma~\ref{lem:est}~\ref{it:est1} and~\ref{it:est2}, we obtain the claimed bounds for $\rho^*$. If $\ell = \omega(1)$, then $\frac{\ln(s+1)}{\ell} = o(1)$ and both bounds coincide apart from lower order terms, that is, we have $\rho^* = (1-o(1)) \frac{\ln(s+1)}{\ell}$. 
\end{proof}

We next show that, for $\ell \ge 1$, the success probability $p_{\text{suc}}(\rho,\ell)=(1-\rho)^\ell$ changes by a factor of $(1 \pm \Omega(\delta))$ when we replace the target mutation rate $\rho^*$ by $\rho^* \pm \delta$. Note that for $\ell = 0$, we have $p_{\text{suc}}(\rho,\ell) = 1$ for all $\rho$.

\begin{lemma}[success probabilities around $\rho^*$]\label{lem:successprob}
  Let $\ell \ge 1$. Let $p_{\text{suc}}(\rho) := p_{\text{suc}}(\rho,\ell) = (1 - \rho)^\ell$ for all $\rho \in [0,1]$. 
  \begin{itemize}
  \item For all $\delta > 0$ such that $(1+\delta)\rho^* \le 1$, we have \[p_{\text{suc}}((1+\delta) \rho^*) \le p_{\text{suc}}(\rho^*) (1 - \tfrac 12 \min\{\delta,\tfrac 1{\ln(s+1)}\} \rho^* \ell).\]
  \item For all $0 < \delta \le 1$, we have 
	\[p_{\text{suc}}((1-\delta) \rho^*) \ge p_{\text{suc}}(\rho^*) (1 + \delta \rho^* \ell).\]  
  \end{itemize}
\end{lemma}

\begin{proof}
   Since $\rho \mapsto p_{\text{suc}}(\rho)$ is non-increasing, we may assume $\delta \le \tfrac{1}{\ln(s+1)}$ to prove the first claim. We then have $\ell\delta\rho^* \le 1$, see~Lemma~\ref{lem:rhostar}, and compute
  \begin{align*}
  \frac{p_{\text{suc}}((1+\delta) \rho^*)}{p_{\text{suc}}(\rho^*)} &= \left(\frac{1-(1+\delta) \rho^*}{1-\rho^*}\right)^\ell = \left(1 - \frac{\delta \rho^*}{1-\rho^*}\right)^\ell \le (1 - \delta \rho^*)^\ell \\
  &\le \exp(-\ell\delta\rho^*) \le 1 - \tfrac 12 \ell \delta \rho^*
  \end{align*}
  using Lemma~\ref{lem:est}~\ref{it:est1} and~\ref{it:est2}.
  
  For the second claim, we simply compute
  \begin{align*}
  \frac{p_{\text{suc}}((1-\delta) \rho^*)}{p_{\text{suc}}(\rho^*)} &= \left(\frac{1-(1-\delta) \rho^*}{1-\rho^*}\right)^\ell = \left(1 + \frac{\delta \rho^*}{1-\rho^*}\right)^\ell \ge (1 + \delta \rho^*)^\ell \ge 1 + \ell \delta \rho^*,
  \end{align*}
  where the last estimate stems from Bernoulli's inequality (Lemma~\ref{lem:est}~\ref{it:estbernoulli}).
\end{proof}

From the previous lemma we will next derive that we have an at least multiplicative drift~\cite{DoerrJW12} towards a small interval around the target rate $\rho^*(\ell,s)$, which allows to prove upper bounds for the time to enter such an interval. For convenience, we show a bound that holds with probability $1 - \frac 1n$ even though we shall later only need a failure probability of $o(1)$.

To ease the analysis of the mutation rate adjustment, we shall here and in a few further lemmas regard the variant of the self-adjusting \oea which, in case it generates an improving solution, does not accept this solution, but instead continues with the parent. It is clear that the mutation rate behaves identical in this variant and in the original self-adjusting \oea until the point when an improving solution is generated. We call this EA the \emph{self-adjusting \oea ignoring improvements}.

\begin{lemma}\label{lem:approachrho}
  Assume that the self-adjusting \oea is started with a search point of fitness $\ell \ge 1$ and with the initial mutation rate $\rho_0 \in [\pmin,\pmax]$ with $\pmin \in o(n^{-1}) \cap \Omega(n^{-c})$ and $\pmax = 1$. Let $\rho^* = \rho^*(\ell,s)$. Let $\delta = \omega(\eps) \cap o(1)$. 
  For 
  \[t := (1+o(1)) \frac{2(\max\{\rho_0 / \rho^*, \rho^*/\rho_0\} + \ln(n))}{\delta \rho^* \ell \eps} = \Theta\left(\frac{\log n}{\delta\eps}\right),\] 
  the time $T^*$ until a search point with higher fitness is generated or the mutation rate $\rho_{T^*}$ is in $[(1-\delta) \rho^*, (1+\delta) \rho^*]$ satisfies 
  \[\Pr\left[T^* \ge t\right] \le \tfrac 1 {n}.\]
  
  For $\ell = 0$, we have that within $\lceil \frac 1s \log_{1+\eps} \frac 1 {\rho_0} \rceil+1 = O(\frac{\log n}{\eps})$ iterations with probability one an improvement is found.
\end{lemma}

\begin{proof}
  For $\ell = 0$, each iteration of the self-adjusting \oea is a success. Consequently, the mutation rate is multiplied by $(1+\eps)^s$ in each iteration until the maximum possible value of $1$ is reached or an improvement is found. Once the mutation rate is $1$, surely the next iteration gives an improvement. This easily gives the claimed result. 
  
  Hence let us concentrate on the more interesting case that $\ell \ge 1$. We shall again use the shorthand $p_{\text{suc}}(\rho) := p_{\text{suc}}(\rho,\ell)$ for all $\rho$. We regard the \oea ignoring improvements instead of the original EA. This does not change the time $T^*$, the first time an improving solution is generated or the rate enters the interval $[(1-\delta) \rho^*, (1+\delta) \rho^*]$. However, it eases the definition of $T := \min\{t \in \N \mid \rho_t \in [(1-\delta) \rho^*, (1+\delta) \rho^*]\}$, which obviously stochastically dominates~$T^*$. Hence we proceed by proving upper bounds for $T$.

  Assume first  
  that $\rho_0 > \rho^+ := (1+\delta) \rho^*$. In this case, since $\delta = \omega(\eps)$ and $n$ is sufficiently large, we have $T = \min\{t \in \N \mid \rho_t \le (1+\delta) \rho^*\}$. By Lemma~\ref{lem:successprob}, we have $p_{\text{suc}}(\rho) \le p_{\text{suc}}(\rho^*) (1 - \tfrac 12 \delta \rho^* \ell)$ for all $\rho \ge \rho^+$. We use this to compute the expected change of the mutation rate. When $\rho_t \ge \rho^+$, then by the statement just derived from Lemma~\ref{lem:successprob} we have 
  \begin{align*}
  E[\rho_{t+1}] &= p_{\text{suc}}(\rho_t) \rho_t (1+\eps)^s + (1-p_{\text{suc}}(\rho_t)) \rho_t (1+\eps)^{-1} \\
  &\le p_{\text{suc}}(\rho^*) (1 - \tfrac 12 \delta \rho^* \ell) \rho_t (1+\eps)^s + (1-p_{\text{suc}}(\rho^*) (1 - \tfrac 12 \delta \rho^* \ell)) \rho_t (1+\eps)^{-1} \\
  &\le \rho_t \left(p_{\text{suc}}(\rho^*) (1 - \tfrac 12 \delta \rho^* \ell) (1 + s\eps + O(\eps^2)) + (1-p_{\text{suc}}(\rho^*) (1 - \tfrac 12 \delta \rho^* \ell)) (1 - \eps + O(\eps^2))\right)\\
  & = \rho_t \left(1 + p_{\text{suc}}(\rho^*)s\eps - (1-p_{\text{suc}}(\rho^*))\eps - p_{\text{suc}}(\rho^*) \cdot \tfrac 12 \delta \rho^* \ell (s+1) \eps + O(\eps^2) \right)\\
  & = \rho_t \left(1 - \tfrac 12 \delta \rho^* \ell \eps + O(\eps^2) \right).
  \end{align*}
Consider now the process $(\tilde\rho_t)$ defined by $\tilde\rho_t = \rho_t$ for $t < T$ and $\tilde\rho_t = 0$ otherwise. We have again $E[\tilde\rho_{t+1} \mid \tilde\rho_t] \le \tilde\rho_t \left(1 - \tfrac 12 \delta \rho^* \ell \eps + O(\eps^2) \right)$ for all $t$. 
With a simple induction, we see that $E[\tilde\rho_t] \le \rho_0 \left(1 - \tfrac 12 \delta \rho^* \ell \eps + O(\eps^2) \right)^t$. From Lemma~\ref{lem:est}~\ref{it:est1} as well as $\delta = \omega(\eps)$ and $\rho^* \ell = \Theta(1)$, we conclude $E[\tilde\rho_t] \le \rho_0 \exp(- (1-o(1)) \tfrac 12 t \delta \rho^* \ell \eps)$. For $t = (1+o(1)) 2 \frac{\ln(\rho_0/\rho^+) + \ln(n)}{\delta \rho^* \ell \eps}$, we have $E[\tilde\rho_t] \le \frac{\rho^+}{n}$ and thus $\Pr[T \ge t] = \Pr[\tilde\rho_t \ge \rho^+] \le \frac 1 {n}$ by Markov's inequality.

The case that $\rho_0 < \rho^- := (1-\delta)\rho^*$ is mostly similar except that we now regard the reciprocal of the rate. Using Lemma~\ref{lem:successprob} we compute, conditional on $\rho_t \le \rho^+$,
\begin{align*}
E[\tfrac{1}{\rho_{t+1}}]
& =  p_{\text{suc}}(\rho_t) \tfrac{1}{\rho_t} (1+\eps)^{-s} + (1-p_{\text{suc}}(\rho_t)) \tfrac{1}{\rho_t} (1+\eps)\\
&\le \tfrac{1}{\rho_t} \left(p_{\text{suc}}(\rho^*)(1+\delta\rho^*\ell) (1+\eps)^{-s} + (1-p_{\text{suc}}(\rho^*)(1+\delta\rho^*\ell))  (1+\eps)\right)\\
&\le \tfrac{1}{\rho_t} \left(p_{\text{suc}}(\rho^*)(1+\delta\rho^*\ell) (1 - s\eps + O(\eps^2)) + (1-p_{\text{suc}}(\rho^*)(1+\delta\rho^*\ell)) (1+\eps) \right)\\
&\le \tfrac{1}{\rho_t} \left(1 - \delta\rho^*\ell\eps + O(\eps^2)\right).
\end{align*} 
Now a drift argument analogous to above shows that for $t = (1+o(1))\frac{\ln(\rho^-/\rho_0)+\ln(n)}{\delta\rho^*\ell\eps}$, we have $\Pr[T \ge t] \le \frac 1 {n}$.
\end{proof}

Next we will show that the mutation rate of the self-adjusting \oea ignoring improvements (as defined before Lemma~\ref{lem:approachrho}) is likely to stay within a small interval around~$\rho^*(\ell,s)$.

\begin{lemma}\label{lem:rhoisgood}
  Let $\delta = o(1)$ be such that $\delta / \ln(1/\delta) = \omega(\eps)$. There is a $\gamma = o(1)$ such that the following is true. Let $\ell \in [1..n]$, $\rho^* := \rho^*(\ell,s)$, and $\rho_0 \in [(1-\delta) \rho^*, (1+\delta) \rho^*]$. Consider a run of the self-adjusting \oea ignoring improvements, started with a search point of fitness $\ell$ and with the initial mutation rate $\rho_0$. Denote the mutation rate after the adjustment made in iteration $t$ by $\rho_t$. Then for any $T = \omega(1)$, with probability $1-o(1)$ we have 
  \[|\{t \in [1..T] \mid  \rho_t \notin [(1-\gamma) \rho^*, (1+\gamma) \rho^*]\}| = o(T).\]
\end{lemma}

\begin{proof}
  Let $\ell \ge 1$ in the remainder. Let $T = \omega(1)$. We first argue that we have $\rho_t \ge (1+\gamma)\rho^*$, for a $\gamma$ made precise below, only for a sub-linear number of the $t \in [1..T]$. 
 
  We consider the random process $X_t := \log_{1+\eps} (\rho_t / \rho^+)$, where $\rho^+ := (1+\delta) \rho^*$. By assumption, $X_0 \le 0$. If $X_t \ge 0$ for some $t$, then $\rho_t \ge (1+\delta) \rho^*$. By Lemma~\ref{lem:successprob}, we have $\rho_{t+1} = \rho_t (1+\eps)^s$ with probability at most $\frac{1}{s+1}(1 - \frac 12 \delta \rho^* \ell) =: p$, and we have $\rho_{t+1} = \rho_t (1+\eps)^{-1}$ otherwise. Consequently, we have $\Pr[X_{t+1} = X_t + s] \le  p$ and $\Pr[X_{t+1} = X_t -1] = 1 - \Pr[X_{t+1} = X_t + s]$. Let $D$ be a random variable taking the value $+s$ with probability $p$ and the value $-1$ with probability $1-p$. Then, regardless of the outcomes of $X_1, \dots, X_t$ (but still assuming $X_t \ge 0$), we have $X_{t+1} \preceq X_t + D$. We compute $E[D] = ps - (1-p) = -\frac 12 \delta \rho^* \ell$ and observe $|D| \le s$. If $X_t < 0$, we still know that $X_{t+1} - X_t$ takes only the values $-1$ and $s$. 
  
  Consequently, the process $(X_t)$ satisfies the assumptions of Lemma~\ref{lem:occprob} (with $\phi = \frac 12 \delta \rho^* \ell$). Taking  $U = 6 \frac{s^2}{\phi} \ln(\frac 1{\phi}) + s$, we have $\Pr[X_{t} \ge U] = o(1)$ for all $t \in [1..T]$. Let $\gamma^+$ be such that $1+\gamma^+ = (1+\delta)(1+\eps)^U$. Note that by our assumptions on $\delta$ and $\eps$, we have $(1+\gamma^+) \le \exp(\delta + U\eps) = \exp(o(1)) = 1 + o(1)$, that is, $\gamma^+ = o(1)$. By construction, $\Pr[\rho_{t} \ge (1+\gamma^+)\rho^*] = \Pr[X_t \ge U] = o(1)$ for all $t \in [1..T]$. By linearity of expectation, $E_0 := E[|\{t \in [1..T] \mid  \rho_t \ge (1+\gamma^+) \rho^*\}|] = o(T)$. Setting somewhat arbitrarily $\nu := \sqrt{E_0/T} = o(1)$, by Markov's inequality we have $|\{t \in [1..T] \mid  \rho_t \ge (1+\gamma^+) \rho^*\}| < \nu T =o(T)$ with probability at least $1-E_0/(\nu T) = 1-\nu = 1-o(1)$.

  We now argue that $\rho_t$ can not be too small too often either. Let $\rho^- = (1-\delta)\rho^*$ and consider the random process $(X_t)$ defined by $X_t = \log_{1+\eps}(\rho^- / \rho_t)$. Since $\rho_0 \ge (1-\delta)\rho^*$, we have $X_0 \le 0$. If $X_t \ge 0$ for some $t$, then $\rho_t \le \rho^-$ and Lemma~\ref{lem:successprob} shows that we have $\rho_{t+1} = \rho_t (1+\eps)^s$ and hence $X_{t+1} = X_t - s$ with probability at least $\frac 1 {s+1} (1 + \delta \rho^* \ell) =: p$. Otherwise, we have $\rho_{t+1} = \rho_t (1+\eps)^{-1}$ and $X_{t+1} = X_t + 1$. Consequently, $X_{t+1}$ is stochastically dominated by $X_t + D$, where $D$ is such that $\Pr[D = -s] = p$ and $\Pr[D = 1] = 1-p$. We have $E[D] = -sp + (1-p) = -\delta\rho^*\ell$ and $|D| \le s$. If $X_t < 0$, we still have that $X_{t+1} - X_t$ is a discrete random variable with values in $[-s,s]$. 

  Consequently, the process $(X_t)$ again satisfies the assumptions of Lemma~\ref{lem:occprob}, now with $\phi = \delta \rho^* \ell$. With  $U = 6 \frac{s^2}{\phi} \ln(\frac 1{\phi}) + s$ and $\gamma^-$ such that $1-\gamma^- = (1-\delta)(1+\eps)^{-U}$, we have $\gamma^- = o(1)$ and $\Pr[\rho_{t} \le (1-\gamma^-)\rho^*] = \Pr[X_{t} \ge U] = o(1)$ for all $t \in [1..T]$. Again by linearity of expectation and Markov's inequality, $|\{t \in [1..T] \mid  \rho_t  \le (1-\gamma^-) \rho^*\}| = o(T)$ with probability $1-o(1)$.

  Taking $\gamma = \max\{\gamma^-,\gamma^+\}$, we have $\gamma = o(1)$ and $|\{t \in [1..T] \mid  \rho_t \notin [(1-\gamma) \rho^*, (1+\gamma) \rho^*]\}| = o(T)$. We note that this definition of $\gamma$ depends on $\ell$. However, the dependence can be expressed as a dependence on $\rho^*\ell$ only. Since $\rho^*\ell = \Theta(1)$ by Lemma~\ref{lem:rhostar}, all the values of $\gamma$ appearing in the above proof for different values of $\ell$ are of the same asymptotic order of magnitude. Hence we can choose $\gamma$ independent of $\ell$. This complete the proof of this lemma.
\end{proof}

We use the following estimate for the improvement probability $p_{\text{imp}}(\rho,\ell) := (1-\rho)^\ell \rho$ of generating a better individual from an individual of fitness $\ell$ via standard-bit mutation with mutation rate $\rho$. 

\begin{lemma}\label{lem:strictimp}
  Let $\ell \in [1..n-1]$ and $\rho^* := \rho^*(\ell,s)$. Let $\gamma = o(1)$ and $\rho \in [(1-\gamma) \rho^*,(1+\gamma) \rho^*]$. Then $p_{\text{imp}}(\rho,\ell) \ge p_{\text{imp}}(\rho^*,\ell) (1 - O(\gamma))$.
\end{lemma}

\begin{proof}
  If $\rho < \rho^*$, then $p_{\text{imp}}(\rho,\ell) = (1-\rho)^\ell \rho > (1-\rho^*)^\ell (1-\gamma) \rho^* = (1-\gamma) p_{\text{imp}}(\rho^*,\ell)$. If $\rho > \rho^*$, then in a similar fashion as in the proof of Lemma~\ref{lem:successprob}, we compute
  \begin{align*}
  \left(\frac{1-(1+\gamma) \rho^*}{1-\rho^*}\right)^\ell & = \left(1 - \frac{\gamma \rho^*}{1-\rho^*}\right)^\ell \ge \exp\left( - \frac{2\gamma \rho^*}{1-\rho^*}\right)^\ell\\
& = \exp\left( - \frac{2\gamma \rho^* \ell}{1-\rho^*}\right)
  \stackrel{\text{Lem.}~\ref{lem:rhostar}}{\ge} \exp\left( - \frac{2\gamma \ln(s+1)}{1-\rho^*}\right)
  \ge 1 -  \frac{2\gamma \ln(s+1)}{1-\rho^*}.
  \end{align*}
  Consequently, 
  \begin{align*}
  p_{\text{imp}}(\rho,\ell) &= (1-\rho)^\ell \rho \ge (1-(1+\gamma) \rho^*)^\ell \rho^* \ge (1 - \rho^*)^\ell \left(1 -  \frac{2\gamma \ln(s+1)}{1-\rho^*}\right) \rho^* \\
  & = \left(1 -  \frac{2\gamma \ln(s+1)}{1-\rho^*}\right) p_{\text{imp}}(\rho^*,\ell).
  \end{align*}
\end{proof}

We now have the necessary prerequisites to show the main ingredient of our running-time analysis: the time needed to leave fitness level $\ell$ is (essentially) at least as good as if the EA would always use the target mutation rate $\rho^*(\ell,s)$. This holds not only with respect to the expectation, but also when regarding distributions.

\begin{lemma}\label{lem:gaintime}
  Let $c$ be a constant and $\pmin \in o(n^{-1}) \cap \Omega(n^{-c})$. Let $\eps = \omega(\frac{\log n}{n}) \cap o(1)$. Let $\delta = o(1)$ be such that $\delta / \ln(\frac{1}{\delta}) = \omega(\eps)$ and $\delta = \omega(\frac{\log n}{n\eps})$. Assume that the self-adjusting \oea is started with a search point of fitness $\ell \in [0..n-1]$ and an arbitrary mutation rate $\rho \ge \pmin$. Let $\rho^* = \rho^*(\ell,s)$. Then the number $T_\ell$ of iterations until a search point with fitness better than $\ell$ is found is stochastically dominated by
  \[T_\ell \preceq o(n) + \Geom(\min\{\omega(\tfrac 1n), (1-o(1))(1-\rho^*)^\ell \rho^*\}).\]
  In particular, $E[T_\ell] \le o(n) + \frac{1}{(1-\rho^*)^\ell \rho^*}$.
\end{lemma}

\begin{proof}
  For $\ell = 0$, Lemma~\ref{lem:approachrho} contains the claimed result. Hence let $\ell \ge 1$. Let $Q := p_{\text{imp}}(\rho^*,\ell) = (1-\rho^*)^\ell \rho^* = \frac{\rho^*}{s+1} = \Theta(\frac{1}{\ell})$ be the probability of finding an improving solution when using the mutation rate $\rho^*$. 
  
  We first show that there is a $t = o(n)$ such that the probability that the EA in the first $t$ iterations does not find an improving solution, is at most $o(1) + (1 - (1-o(1))Q)^t$.
  
  By Lemma~\ref{lem:approachrho}, there is a $t_0 = O(\frac{\log n}{\delta \eps}) \subseteq o(n)$ such that with probability at least $1 - \frac 1n$ within the first $t_0$ iterations a $\rho$-value in $[(1-\delta)\rho^*, (1+\delta) \rho^*]$ is reached or an improvement is found.
  
  Assume that after $t_0$ iterations we have not found an improvement (otherwise we are done) and that the first time $T_0$ such that the mutation rate is in $[(1-\delta)\rho^*, (1+\delta) \rho^*]$ is at most $t_0$. Let $\gamma = o(1)$ as in Lemma~\ref{lem:rhoisgood}. Let $t_1 \in \omega(t_0) \cap o(n)$, and to be more concrete, let $t_1 = t_0^{2/3} n^{1/3}$. By Lemma~\ref{lem:rhoisgood}, with probability $1-o(1)$, in all but a lower-order fraction of the iterations $[T_0+1..T_0+t_1]$ the algorithm ignoring improvements uses a mutation rate in $[(1-\gamma)\rho^*, (1+\gamma) \rho^*]$. By Lemma~\ref{lem:strictimp}, for any such rate $\rho$ the probability $p_{\text{imp}}(\rho,\ell) = (1-\rho)^\ell \rho$ of finding an improvement is at least $(1 - O(\gamma)) Q$. Let us assume for a moment that indeed the mutation rate is in this range for a $1-o(1)$ fraction of the iterations $T_0+1, \ldots, T_0+t_1$. Then the probability of not finding an improvement in any of these iterations is at most 
  \[(1 - (1 - O(\gamma)) Q)^{t_1 (1 - o(1))} \le  (1 - (1 - o(1)) (1 - O(\gamma)) Q)^{t_1} = (1 - (1 - o(1)) Q)^{t_1} =: P,\] 
  where the inequality stems from Lemma~\ref{lem:est}~\ref{it:estbernoulli2}.
  
  Let $t = t_0 + t_1$, which is still $o(n)$. Since $t = t_1 (1+o(1))$, again by Lemma~\ref{lem:est}~\ref{it:estbernoulli2}, we have $P = (1 - (1 - o(1))Q)^t$. Taking now also into account the two failure probabilities of order $o(1)$, we have shown that the probability of not finding an improvement in the first $t$ iterations is at most $o(1) + (1 - (1 - o(1))Q)^t =: \bar P$.
  
  If $Q \le \frac 1t$, then $\bar P = (1+o(1))(1 - (1 - o(1))Q)^t = (1 - (1 - o(1))Q)^t$. If $Q > \frac 1t$, then we redo the above construction with $t_1 = t_0^{1/3} n^{2/3}$. We can do so since we have never exploited the particular size of $t$. Now $Q = \omega(1/t)$ and consequently, $\bar P = o(1)$.

  We now repeat such phases of $t$ iterations. Note that if such a phase of $t$ iterations does not lead to an improvement, then we are in the same situation as initially. Hence the probability that $k$ such phases do not lead to an improvement is at most $\bar P^k$. 
  
  If $Q \le \frac 1t$ and hence $\bar P = (1 - (1 - o(1))Q)^t$, the probability that $k \ge t$ iterations do not lead to an improvement is at most $((1 - (1 - o(1))Q)^t)^{\lfloor k/t \rfloor} \le ((1 - (1 - o(1))Q)^t)^{(k/t)-1} = (1 - (1 - o(1))Q)^{k-t}$. Consequently, the time $T$ to find an improvement is stochastically dominated by $t+\Geom((1 - o(1))Q)$. 
  
  In the other case that $\bar P = o(1)$, the probability that $k \ge t$ iterations do not lead to an improvement is at most $\bar P^{\lfloor k/t \rfloor} \le \bar P^{(k/t)-1} = (\bar P^{1/t})^{k-t}$. Since $\bar P = o(1)$, we have $\bar P^{1/t} \le 1 - \frac 1t = 1 - \omega(\frac 1n)$. Consequently, the time $T$ to find an improvement is stochastically dominated by $t+\Geom(\omega(\frac 1n))$.  
  
  In particular, we obtain $E[T_\ell] \le o(n) + \frac{1+o(1)}{(1-\rho^*)^\ell \rho^*}$, and by Lemma~\ref{lem:rhostar} the $o(1)$-term gets swallowed by the $o(n)$ term.
\end{proof}

Having shown this bound for the time needed to leave each fitness level, we can now derive from it a bound for the whole running time. In principle, Wegener's fitness level method~\cite{Wegener01} would be an appropriate tool here, since it -- essentially -- states that the expected running time is the sum of the expected times needed to leave each fitness level. For the \leadingones function, however, it has been observed that many algorithms visit each fitness level only with probability $1/2$, so by simply using the fitness level method we would lose a factor of two in the running-time guarantee. Since we believe that our running-time results are tight apart from constant factors, we have to care about this factor of two.

The first result in this direction is the precise running-time analysis of the \oea with static and fitness-dependent mutation rates on \leadingones in~\cite{BottcherDN10}. The statement that the running time is half of the sum of the exit times of the fitness levels was stated (for expected times) before Theorem~3 in~\cite{BottcherDN10}, but a formal proof (which could easily be obtained from Theorem~2 there) was not given. In~\cite[Theorem~4]{Sudholt13}, a refinement of the fitness-level method was developed. It yields upper bounds for expected running times which can be below the sum of the improvement times. However, it requires a careful choice of the method parameters $\chi, (s_i)_{i \in [0..m-1]}$, and $(\gamma_{i,j})_{i < j}$. With such a choice, the running time of a broad class of evolutionary algorithms on \leadingones was determined to be the sum of the improvement times divided by two~\cite[Section~V]{Sudholt13}. A significantly simpler version of the fitness level method was recently presented in~\cite{DoerrK21gecco}. It was used to give an elementary proof of the fact that the expected running time of the \oea with general mutation rate on \leadingones is half the sum of the improvement times.

Since we also aim at a stochastic domination result, these results only working with expected running times are not applicable. A very general result based on stochastic domination was presented and formally proven in~\cite{Doerr19tcs}. Unfortunately, this result was formulated only for algorithms using the same mutation operator in all iterations spent on one fitness level since this implies that the time to leave a fitness level follows a geometric distribution. This result is thus not applicable to our self-adjusting \oea. By a closer inspection of the proof, we observe that the restriction to using the same mutation operator in all iterations on one fitness level is not necessary when the result is formulated via geometric distributions. We phrase the resulting theorem in full generality, i.e., for all unbiased mutations operators (in the sense introduced by Lehre and Witt~\cite{LehreW12} -- we do not present details about this concept, but for the purpose of this paper it suffices to know that the mutation operators of the \oea and the \oeares satisfy this condition). 

\begin{theorem}
\label{thm:level}
  Consider a \oea which may use in each iteration a different unbiased mutation operator. This choice may depend on the whole history. Consider that we use this algorithm to optimize the \leadingones function. For each $\ell \in [0..n-1]$ let $T_\ell$ be a random variable that, regardless of how the algorithm reached this fitness level, stochastically dominates the time the algorithm takes to go from a random solution with fitness exactly $\ell$ to a better solution. Then the running time $T$ of this \oea on the \leadingones function is stochastically dominated by \[T \preceq \sum_{\ell = 0}^{n-1} X_\ell T_\ell,\] where the $X_\ell$ are uniformly distributed binary random variables and all $X_\ell$ and $T_\ell$ are independent. In particular, the expected running time satisfies \[E[T] \leq \frac 12 \sum_{\ell = 0}^{n-1} E[T_\ell].\]
\end{theorem}
  
\begin{proof}
  In the proof of Theorem~3 in the full version of~\cite{Doerr19tcs}, equation~(1), which is $T_i^0 = \Geom(q_i) + T_{i+1}^{\rand}$, is also valid in the form $T_i^0 = T_i + T_{i+1}^{\rand}$ in our setting, and any following occurrence of $\Geom(q_i)$ can be replaced by $T_i$. This proves our result.
\end{proof}
  
Now Lemma~\ref{lem:gaintime} and Theorem~\ref{thm:level} easily yield our running-time bound for the self-adjusting \oea stated in Theorem~\ref{thm:oea}. Both the domination and the expectation version of this running-time bound are, apart from lower order terms, identical to the bounds which could easily be shown for the \oea which uses a fitness-dependent mutation rate of $\rho(f(x)) := \rho^*(f(x),s)$ when the parent has fitness $f(x)$. This indicates that our self-adjustment tracks the target mutation rate $\rho^*(f(x),s)$ very well. 

\begin{proof}[Proof of Theorem~\ref{thm:oea}]
  Choose $\delta \in o(1)$ such that $\delta / \ln(\frac 1\delta) = \omega(\eps)$ and $\delta = \omega(\frac{\log n}{n \eps})$. Note that such a $\delta$ exists, e.g., $\delta = \max\{\sqrt{\eps}, \sqrt{\frac{\log n}{n\eps}}\}$. Now Lemma~\ref{lem:gaintime} gives upper bounds for the times $T_\ell$ to leave the $\ell$-th fitness level, which are independent of the mutation rate present when entering the fitness level. Hence by Theorem~\ref{thm:level}, the required stochastic dominance follows.
	
	For the last claim~\eqref{eq:Toeagen} in Theorem~\ref{thm:oea}, the first inequality is an immediate consequence of the domination statement. For the second one, we use the bound $\rho^* = (1 - o(1)) \frac{\ln(s+1)}{\ell}$ from Lemma~\ref{lem:rhostar}. This implies in particular that for $\ell=\omega(1)$ we have $(1-\rho^*)^\ell = (1-o(1))e^{-\rho^*\ell} = (1-o(1))\cdot 1/(s+1)$. The second step in~\eqref{eq:Toeagen} then follows by plugging in.
\end{proof}

\section{The Self-Adjusting \texorpdfstring{\oeares}{(1+1) EA}}
\label{sec:oeares}

We now extend our findings for the \oea to the \oeares, the resampling variant of the \oea which enforces that offspring are different from their parents by ensuring that at least one bit is flipped by the standard bit mutation operator. That is, the \oeares differs from the \oea only in the choice of the mutation strength $k$, which in the \oea follows the binomial distribution $\Bin(n,p)$, and in the \oeares follows the conditional binomial distribution $\Bin_{>0}(n,p)$ which assigns every positive integer $1 \le m \le n$ a probability of $\binom{n}{m}p^m(1-p)^{n-m}/(1-(1-p)^n)$. The self-adjusting version of the \oeares implements the same change, and can thus be obtained from Algorithm~\ref{alg:adaoea} by exchanging line~\ref{line:k} by ``Sample $k$ from $\Bin_{>0}(n,p)$''. Note that this is also the algorithm empirically studied in~\cite{DoerrW18}.

It is clear that for static mutation rates the \oeares is strictly better than the plain \oea, since it simply avoids the useless iterations in which duplicates of the parent are evaluated. For example, it reduces the running time of the \oea with static mutation rate $1/n$ by a multiplicative factor of $(e-1)/e \approx 0.632$~\cite{CarvalhoD18}. For the self-adjusting \oeares, however, it is \emph{a priori} not evident how the conditional sampling of the mutation strengths influences the running time. Note that after each iteration in which no bit is flipped by the \oea (e.g., a $(1-1/n)^n \approx 1/e$ fraction of iterations for mutation rate $1/n$), the mutation rate is increased by the factor $F^s$. Since these steps are avoided by the self-adjusting \oeares, it could, in principle, happen that the actual mutation rates are smaller than what they should be. We show in this section that this is not the case. Put differently, we show that the self-adjusting \oeares also achieves very efficient optimization times. In contrast to the results proven in Section~\ref{sec:oea}, however, we derive a bound that is difficult to evaluate in closed form. For interpreting this bound, we therefore have to resort to a numerical evaluation of the otherwise formally proven bound. A comparison with the best possible \oeares variant using optimal fitness-dependent mutation rates will show that the obtained running times are very similar, see Section~\ref{sec:numerics-oeares}. 

Before we can state the main theorem of this section, Theorem~\ref{thm:oeares}, we first need to discuss how the conditional sampling of the mutation strengths influences the improvement and the success probabilities. 
It is not difficult to see that the \emph{improvement probability} $\hat p_{\text{imp}}(\rho,\ell)$ of the \oeares, started in an arbitrary search point $x$ with $\LO(x)=\ell$ and using mutation rate~$\rho$, equals 
\begin{align}\label{eq:phat}
\hat p_{\text{imp}}(\rho,\ell) 
= \frac{(1-\rho)^{\ell}\rho}{1-(1-\rho)^n},
\end{align}
which is the improvement probability of the \oea divided by the probability that the unconditional standard bit mutation does not create a copy of its input. 

Likewise, the \emph{success probability} $\hat p_{\text{suc}}(\rho,\ell)$ 
of the \oeares in the same situation 
can be computed as 
\begin{equation}
\label{eq:qhat}
\hat p_{\text{suc}}(\rho,\ell) 
= \frac{(1-\rho)^{\ell}(1-(1-\rho)^{n-\ell})}{1-(1-\rho)^n} 
= 1-\frac{1-(1-\rho)^\ell}{1-(1-\rho)^n},
\end{equation}
where the probability in the numerator is given by the probability of not flipping one of the first $\ell$ bits times the probability to flip at least one bit in the last $n-\ell$ positions (recall here that flipping 0 bits is not possible with the \oeares).

As we did for the \oea, we would like to define for the \oeares a target mutation rate $\hat \rho^*(\ell,s)$ to be the one that guarantees that the success probability equals $1/(s+1)$. That is, we would like to set the target rate as the value of $\hat \rho^*$ that solves the equation 
\begin{equation}\label{eq:rhostarhat}
\hat p_{\text{suc}}(\hat \rho^*,\ell) = 1/(s+1).
\end{equation}
However, while the corresponding equation for the \oea has always a (unique) solution, we will show in Lemma~\ref{lem:rhostarhat} that Equation~\eqref{eq:rhostarhat} has a solution only if $\ell < sn/(s+1)$. The reason is that the mutation operator of the \oeares always flips at least one bit, and for the offspring to be accepted, this flip needs to be in the non-optimized tail of the string. As the algorithm progresses, the size of the tail decreases and thus the chances of making a successful mutation decreases as well. In the extreme case of $11\ldots 110$, where only the very last bit is incorrect, the probability of finding this bit is $1/n$, even if we condition on flipping exactly one bit. In contrast, the \oea always has the ``option'' not to flip any bits at all, which is a success by definition. Thus the \oea can achieve success probabilities arbitrarily close to one by making the mutation rate small enough, while the \oeares cannot exceed a mutation rate of $\Theta(1/n)$ for the last steps of optimization. 

We therefore have two phases in which the \oeares behaves differently. For $\ell < sn/(s+1)$, our analysis follows closely the one for \oea, except that the algebra gets considerably more involved. That is, for the regime $\ell < sn/(s+1)$, we define the \emph{target mutation rate} $\hat \rho^*(\ell,s)$ via Equation~\eqref{eq:rhostarhat} and we show that the mutation rate $\rho$ used by the \oeares approaches $\hat \rho^*(\ell,s)$ quickly, and stays close to $\hat \rho^*(\ell,s)$ until a new fitness level is reached. For $\ell \ge sn/(s+1)$, the mutation rate $\rho$ has negative drift, and quickly reaches values $o(1/n)$. In this regime, the \oeares mimics Randomized Local Search (RLS), which flips exactly one bit in each round. Indeed, we will show that the time to leave a fitness level is essentially $\Geom(1/n)$-distributed, as we would expect for the \rls. For technical reasons, we will set the threshold between the two regimes not at $\ell = sn/(s+1)$, but at a slightly smaller value $\ell_0$. This trick helps to avoid some border cases. More precisely, throughout this section we fix a positive function $\eta_0 = \eta_0(n) = o(1)$ and $\ell_0 \in [0,n]$ such that
\begin{equation}\label{eq:ell0}
\ell_0 := (1-\eta_0)\frac{sn}{s+1}, \qquad \text{ where }\quad \eta_0 = o(1).
\end{equation}

With these preparations, the main result can be stated as follows.

\begin{theorem}
\label{thm:oeares}
  Let $c > 1$ be a constant. We consider the self-adjusting \oeares with update strength $F=1+\eps$ for $\eps \in \omega(\frac{\log n}{n}) \cap o(1)$, with minimal mutation rate $\pmin \in o(n^{-1}) \cap \Omega(n^{-c})$, maximal mutation rate $\pmax =1$, and with arbitrary initial mutation rate $\rho_{0} \in [\pmin,\pmax]$. Let $\eta_0 :=\max\{\eps^{1/6}, (\eps n/\log n)^{-1/2}\}$, and let $\ell_0 := \lfloor (1-\eta_0)sn/(s+1) \rfloor$. Then the number $T$ of iterations until the self-adjusting \oeares finds the optimum of the $n$-dimensional \leadingones function satisfies
  \begin{align*}
  T  \preceq o(n^2) & + \sum_{\ell = 0}^{\ell_0} X_\ell \Geom\left(\min\left\{\omega(\tfrac 1n), (1-o(1))\frac{1-(1-\hat\rho^*(\ell,s))^n}{(1-\hat\rho^*(\ell,s))^{\ell}\hat\rho^*(\ell,s)}\right\} \right) \\
&  + \sum_{\ell = \ell_0+1}^{n} X_\ell \Geom((1-o(1))/n),
  \end{align*}
  where the $X_\ell$ are uniformly distributed binary random variables and all $X_\ell$ and all geometric random variables are mutually independent. Furthermore, all asymptotic notation is solely with respect to $n$ and can be chosen uniformly for all $\ell$. In particular, it holds that 
  \begin{align}
	\label{eq:Toearesgen}
  E[T] & \le (1+o(1))\frac 12\left(\frac{n^2}{s+1}+\sum_{\ell=0}^{\ell_0} \frac{1-(1-\hat\rho^*(\ell,s))^n}{(1-\hat\rho^*(\ell,s))^{\ell}\hat\rho^*(\ell,s)}\right).
  \end{align}
  \end{theorem}

\subsection{Numerical Evaluation of the Running-Time Bound in Theorem~\ref{thm:oeares}}
\label{sec:numerics-oeares}

As mentioned above, the interpretation of the running-time bound~\eqref{eq:Toearesgen} is not as straightforward as the corresponding one of the unconditional \oea. For a proper evaluation, one would have to compute bounds on $\hat\rho^*(\ell,s)$, and then plug these into the running-time bound. Since these computations are quite tedious, we will content ourselves with a numerical approximation of $\hat\rho^*(\ell,s)$ and its corresponding running time. The code and selected results of our numerical computations are available online from our GitHub repository at \url{https://github.com/CarolaDoerr/2019-LO-SelfAdjusting}.

Before estimating $E[T]$, we briefly discuss the \oearesopt, the \oeares variant that uses in each round the mutation rate $\presopt(\LO(x))$ which maximizes the improvement probability~\eqref{eq:phat}. The performance of this algorithm is a lower bound for the performance of any \oeares variant with fitness-dependent mutation rates, and thus for our self-adjusting \oeares. We again do not compute $\presopt(\ell)$ exactly, but only numerically. For $n \in \{100, 1\,000, 10\,000\}$ and for all $0 \le \ell<n/2$, the numerically computed values are quite close, but not identical to $1/(\ell+1)$; for $n=10\,000$ the largest difference between $\presopt(\ell)$ and $1/(\ell+1)$ is $0.0001741$ and the smallest is $-0.0000382$. In Lemma~\ref{lem:estimateshatrhostar} below, we will formally prove that $\presopt(\ell) = \Theta(1/\ell)$ in this range. For $\ell \ge n/2$, it is not difficult to see that $\presopt(\ell)$ is obtained by the limit at $0$, in which case the \oeares reduces to RLS. The expected running time of \oearesopt is 
\[
1+\frac{1}{2}\sum_{\ell=0}^{n-1}{\min\left\{n,\frac{1-(1-\presopt(\ell))^n}{\presopt(\ell) (1-\presopt(\ell))^{\ell}}\right\}}
\]
For $n=100$ ($n=1\,000$, $n=10\,000$) this expression evaluates to approximately $0.4077 n^2$ ($0.4026n^2$, $0.4027n^2$). As a side remark, we note that the expected running time of the best possible unary unbiased algorithm with fitness-dependent mutation strength distributions for these problem dimensions has an expected running time of around $0.3884n^2$~\cite[Section~3.2]{DoerrW18}. The \oearesopt is thus only around $3.7\%$ worse than this RLS$_{\text{opt}}$ heuristic. Put differently, the cost of choosing the mutation rates from $\Bin_{>0}(n,p)$ instead of deterministically using the optimal fitness-dependent mutation strength is only $3.7\%$. For comparison, we recall that the (unconditional) \oea variant using optimal mutation rates has an expected normalized running time of $e/4 \approx 0.6796$, which is about $75\%$ worse than that of RLS$_{\text{opt}}$.

We now estimate how close the performance of the self-adjusting \oeares gets to this \oearesopt. To this end, we fix $n=10\,000$ and compute $\hat\rho^*(\ell,s)$ for different success ratios $s$. The expected running times, normalized by the factor $1/n^2$, are plotted in Figure~\ref{fig:oearesbyrate}. The interesting region of success ratios between $1.2$ and $1.4$ is plotted in the zoom on the right. For this $n$, the best success ratio is around $1.285$, which gives a normalized expected running time of around $0.403792$. This value is only $0.26\%$ larger than the expected running time of the \oearesopt for $n=10\,000$. A numerical evaluation for $n=50\,000$ shows that the optimal success rate is again around $1.285$, giving a normalized expected running time slightly less than $0.40375375$.

\begin{figure}[t]
\centering
\includegraphics[width=\linewidth]{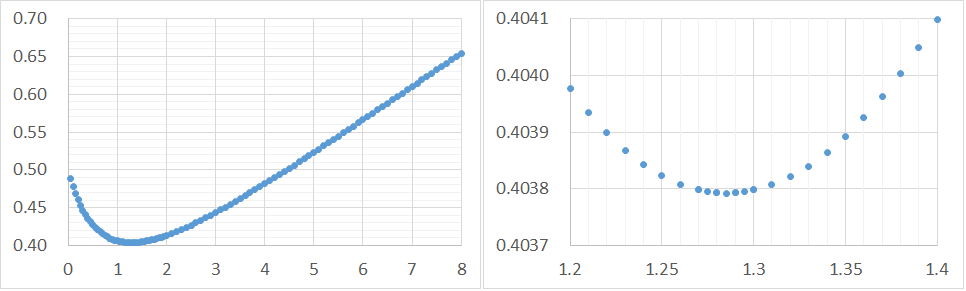}
\caption{Optimization times, normalized by the factor $1/n^2$, of the self-adjusting \oeares on the $10\,000$-dimensional \leadingones function for different success ratios $s$ and assuming an update strength $F=1+o(1)$. The chart on the right zooms into the interesting region around the optimal success ratio.}
\label{fig:oearesbyrate}
\end{figure}

\subsection{Proof of Theorem~\ref{thm:oeares}}
\label{sec:proofoeares} 

We start our proof of Theorem~\ref{thm:oeares} with an elementary lemma, which discusses when Equation~\eqref{eq:rhostarhat} has a solution. 
\begin{lemma}\label{lem:rhostarhat}
For every $0 < \ell < n$, the function $\hat p_{\text{suc}}(\rho,\ell)$ given by~\eqref{eq:qhat} is strictly decreasing in $\rho$ in the range $\rho \in (0,1]$, and its extremal values are given by $\lim_{\rho \to 0} \hat p_{\text{suc}}(\rho,\ell) = 1-\ell/n$ and $\hat p_{\text{suc}}(1,\ell) = 0$. 

In particular, if $\ell < sn/(s+1)$ then \eqref{eq:rhostarhat} has a unique solution $\hat\rho^*(\ell,s)$. If $\ell \geq sn/(s+1)$ then \eqref{eq:rhostarhat} does not have a solution.
\end{lemma}

\begin{proof}
By Lemma~\ref{lem:auxmonotone}, for all $0<b<c$ the function $f(x) = (1-b^x)/(1-c^x)$ is strictly decreasing in $x \in \R^+$. Hence, if we compare $f(x)$ for $x = \ell$ and $x=n$ then we obtain $(1-b^\ell)/(1-c^\ell) > (1-b^n)/(1-c^n)$, or equivalently $(1-b^\ell)/(1-b^n) > (1-c^\ell)/(1-c^n)$ for all $0<b<c<1$. Thus we have shown that the function $g(p) := (1-p^{\ell})/(1-p^{n})$ is strictly decreasing for $0<p<1$, and monotonicity of $\hat p_{\text{suc}}$ follows since it can be expressed via $g$ as
\[
\hat p_{\text{suc}}(\rho,\ell) 
= 1-\frac{1-(1-\rho)^{\ell}}{1-(1-\rho)^n} = 1-g(1-\rho).
\]
For the extremal value at $\rho = 1$, we simply plug in and evaluate $\hat p_{\text{suc}}(1,\ell) = 0$. For $\rho \to 0$, we use L'H\^opital's rule to compute
\begin{align*}
\lim_{\rho \to 0}\hat p_{\text{suc}}(\rho,\ell) 
& = 1-\lim_{\rho \to 0} \frac{1-(1-\rho)^{\ell}}{1-(1-\rho)^n} = 1-\lim_{\rho \to 0} \frac{\frac{\partial}{\partial \rho}(1-(1-\rho)^{\ell})}{\frac{\partial}{\partial \rho}(1-(1-\rho)^n)} = 1-\lim_{\rho \to 0} \frac{\ell(1-\rho)^{\ell-1}}{n(1-\rho)^{n-1}} \\
& = 1-\frac{\ell}{n}.
\end{align*}
The (non-)existence and uniqueness of $\hat\rho^*$ follows immediately from the monotonicity of $\hat p_{\text{suc}}$, since $\ell < sn/(s+1)$ holds if and only if $1/(s+1) > 1-\ell/n$, i.e., if and only if the equation $\hat p_{\text{suc}}(\rho,\ell) = 1/(s+1)$ has a solution. 
 \end{proof}

Lemma~\ref{lem:rhostarhat} tells us when a target mutation rate $\hat\rho^*$ exists. The following lemma quantifies $\hat\rho^*$ up to constant factors. This is a slightly less precise analogue of Lemma~\ref{lem:rhostar}, where we could obtain $\rho^*$ up to $(1+o(1))$ factors.

\begin{lemma}[estimates for $\hat\rho^*$]\label{lem:estimateshatrhostar}
  For $1 \leq \ell < sn/(s+1)$ let $\hat\rho^* = \hat\rho^*(\ell,s)$ be the target mutation rate, i.e., the unique solution of \eqref{eq:rhostarhat}. Then $\hat\rho^*(\ell,s)$ is strictly decreasing in $\ell$. Moreover, if $n \geq C$ for some sufficiently large constant $C = C(s)$, then $\hat p_{\text{suc}}(\rho,\ell)$ satisfies the following bounds. 
  \begin{itemize}
  \item Assume that $\ell = (1-\eta)sn/(s+1)$ for some $0 < \eta \leq 1/(8(s+1)^2)$. Then 
  \[
  \frac{\eta}{n} \leq \hat\rho^*(\ell,s) \leq \frac{4\eta(s+1)}{n} . 
  \]
  \item Assume that $\ell \leq (1-\tfrac{1}{8(s+1)^2})\cdot sn/(s+1)$. Let $\kappa = \kappa(s) := \tfrac14\ln(s+1)/\ln(1+\sqrt{s+1})$. Then 
  \[
  \frac{\min\{\kappa/(8(s+1)^2), \tfrac14\ln(s+1)\}}{\ell} \leq \hat\rho^*(\ell,s) \leq \frac{\max\{s\ln(s+1)/((s+1)\kappa),\ln(s+1)\}}{\ell}. 
  \]
  \end{itemize}
In particular, by definition of $\ell_0 = (1-\eta_0)sn/(s+1)$, for all $0<\ell \leq \ell_0$ we have $\hat\rho^*(\ell,s) \geq \hat\rho^*(\ell_0,s) \geq \eta_0/n$. In the first case we have $\hat\rho^*(\ell,s) = \Theta(\eta/n)$, and in the second case we have $\hat\rho^*(\ell,s) = \Theta(1/\ell)$, where the hidden constants only depend on $s$. In particular, in the first case $1-(1-\hat\rho^*)^\ell = \Theta(\eta)$ and $1-(1-\hat\rho^*)^n = \Theta(\eta)$, and in the second case $1-(1-\hat\rho^*)^\ell = \Theta(1)$ and $1-(1-\hat\rho^*)^n = \Theta(1)$, with hidden constants that only depend on $s$. 
\end{lemma}

\begin{proof}
The monotonicity of $\hat\rho^*$ follows since it is defined as the solution of $\hat p_{\text{suc}}(\rho,\ell) = 1/(s+1)$, and $\hat p_{\text{suc}}$ is strictly decreasing in $\rho$ by Lemma~\ref{lem:rhostarhat}, and strictly decreasing in $\ell$ by \eqref{eq:qhat}. For the bounds, assume first that $\ell = (1-\eta)sn/(s+1)$ for some $0 < \eta \leq 1/(8(s+1)^2)$. Since $\hat p_{\text{suc}}(\rho,\ell)$ is strictly decreasing in $\rho$, it suffices to show that $\hat p_{\text{suc}}(4\eta(s+1)/n,\ell) \leq 1/(s+1) \leq \hat p_{\text{suc}}(\eta/n,\ell)$. Then the solution $\hat\rho^*(\ell,s)$ of the equation $p_{\text{suc}}(\rho,\ell) = 1/(s+1)$ must lie in the interval $[\eta/n, 4\eta(s+1)/n]$. In fact, for $\hat p_{\text{suc}}(4\eta(s+1)/n,\ell)$ we will show for later reference the slightly stronger statement
\begin{align}\label{eq:estimateshatrhostar1a}
\hat p_{\text{suc}}(4\eta(s+1)/n,\ell) \leq \frac{1}{s+1}\left(1-\frac{\eta s}{10}\right).
\end{align}
Let us first note in a preparatory computation by bounding all positive higher order terms with zero that
\begin{align}\label{eq:estimateshatrhostar1}
(1-\eta)(1-&2\eta  s\tfrac{\ell-1}{\ell})(1+2\eta(s+1)\tfrac{n-1}{n} - \tfrac{8}{3}\eta^2(s+1)^2\tfrac{(n-1)(n-2)}{n^2}) \nonumber \\
& \geq 1 + \eta - 2(s+1)(\tfrac{10}{3}s+\tfrac{7}{3})\eta^2  - \frac{16s(s+1)^2}{3}\eta^4 +O(\eta/\ell) + O(\eta/n) > 1 + \frac{\eta}{10},
\end{align}
where the last step follows from $\eta \leq 1/(8(s+1)^2)$ if $n$ (and thus $\ell = \Theta(n)$) is  sufficiently large. Now we use Lemma~\ref{lem:est}\ref{it:estsecondorder} and~\ref{it:estthirdorder} to estimate
\begin{align*}
\hat p_{\text{suc}}(4\eta(s+1)/n,\ell) & = 1- \frac{1-(1-4\eta(s+1)/n)^{\ell}}{1-(1-4\eta(s+1)/n)^{n}} \\
& \stackrel{\ref{lem:est}\ref{it:estsecondorder},\ref{it:estthirdorder}}{\leq} 1- \frac{4\eta(s+1)\ell/n - \binom{\ell}{2}(4\eta(s+1)/n)^2}{4\eta(s+1)-\binom{n}{2}(4\eta(s+1)/n)^2+\binom{n}{6}(4\eta(s+1)/n)^3} \\
&  \stackrel{(s+1)\ell/n \leq s}{\leq} 1- \frac{\ell}{n}\cdot \frac{1-2\eta s\tfrac{\ell-1}{\ell}}{1-2\eta(s+1)\tfrac{n-1}{n} + \tfrac{8}{3}\eta^2(s+1)^2\tfrac{(n-1)(n-2)}{n^2}} \\
&  \stackrel{1/(1-x) \geq 1+x}{\leq} 1-\frac{\ell}{n}(1-2\eta s\tfrac{\ell-1}{\ell})(1+2\eta(s+1)\tfrac{n-1}{n} - \tfrac{8}{3}\eta^2(s+1)^2\tfrac{(n-1)(n-2)}{n^2})\\
&  = 1-\frac{s}{s+1}(1-\eta)(1-2\eta s\tfrac{\ell-1}{\ell})(1+2\eta(s+1)\tfrac{n-1}{n} - \tfrac{8}{3}\eta^2(s+1)^2\tfrac{(n-1)(n-2)}{n^2}) \\
&  \stackrel{\eqref{eq:estimateshatrhostar1}}{\leq} 1-(1+\eta/10)\frac{s}{s+1} = \frac{1}{s+1}\left(1-\frac{\eta s}{10}\right) \leq \frac{1}{s+1}.
\end{align*}
Note that the intermediate step also shows the stronger statement~\eqref{eq:estimateshatrhostar1a}. For $\hat p_{\text{suc}}(\eta/n,\ell)$, we use a similar calculation, but with inequalities more in our favor. This time we may simply use Lemma~\ref{lem:est}~\ref{it:estbernoulli} and~\ref{it:estsecondorder}:
\begin{align*}
\hat p_{\text{suc}}(\eta/n,\ell) 
& = 1- \frac{1-(1-\eta/n)^{\ell}}{1-(1-\eta/n)^{n}} \geq 1- \frac{\eta\ell/n }{\eta-\tfrac{1}{2}\eta^2} = 1- \frac{s}{s+1}\cdot \frac{1-\eta}{1-\eta/2} \geq 1- \frac{s}{s+1}\\ 
& = \frac{1}{s+1}.
\end{align*}
This concludes the first bullet point. For the other case, 
we distinguish two subcases for $\ell$. Assume first that $\ell \leq \kappa n$. Note that this implies 
\begin{align}\label{eq:estimateshatrhostar3}
\exp\left(-\frac{n\ln(s+1)}{4\ell} \right)\leq \exp\left(-\ln\bigg(\frac{s}{\sqrt{s+1}-1}\bigg)\right) = \frac{\sqrt{s+1}-1}{s}.
\end{align}
Hence, we obtain for $\rho = \ln(s+1)/(4\ell)$,
\begin{align*}
1- \hat p_{\text{suc}}\left(\frac{\ln(s+1)}{4\ell},\ell\right) & = \frac{1-(1-\ln(s+1)/(4\ell))^{\ell}}{1-(1-\ln(s+1)/(4\ell))^{n}} \stackrel{\text{Lem.} \ref{lem:est}\ref{it:est1},\ref{it:est3}}{\leq} \frac{1-\exp\big(-\tfrac12\ln(s+1)\big)}{1-\exp\big(-n\ln(s+1)/(4\ell)\big)} \\
& \stackrel{\eqref{eq:estimateshatrhostar3}}{\leq} \frac{1-1/\sqrt{s+1}}{1-(\sqrt{s+1}-1)/s} = \frac{(s+1-\sqrt{s+1})/(s+1)}{(s+1-\sqrt{s+1})/s} = \frac{s}{s+1}.
\end{align*}
Hence, $\hat p_{\text{suc}}\left(\frac{\ln(s+1)}{4\ell},\ell\right) \geq 1/(s+1)$, and since $\hat p_{\text{suc}}(\rho,\ell)$ is decreasing in $\rho$, this implies $\hat \rho^* \geq \ln(s+1)/(4\ell)$. Since the minimum in the lemma can only be smaller, this yields the first inequality in this subcase. For the second inequality, we plug in $\rho = \ln(s+1)/\ell$ and obtain
\begin{align*}
1- \hat p_{\text{suc}}\left(\frac{\ln(s+1)}{\ell},\ell\right) & = \frac{1-(1-\ln(s+1)/\ell)^{\ell}}{1-(1-\ln(s+1)/\ell)^{n}} \stackrel{\text{Lem.} \ref{lem:est}\ref{it:est1}}{\geq} \frac{1-\exp\big(-\ln(s+1)\big)}{1} = \frac{s}{s+1},
\end{align*}
and analogously as before we may conclude that $\hat \rho^* \leq \ln(s+1)/\ell$, which implies the second inequality. This concludes the subcase $\ell \leq \kappa n$. Keep in mind that for this subcase we have shown the stronger statement $\hat \rho^* \leq \ln(s+1)/\ell$, since we will use this bound for the remaining case.

So let us turn to the last remaining case, $\kappa n \leq \ell \leq (1-\tfrac{1}{8(s+1)^2}) sn/(s+1)$. Since $\hat \rho^*(\ell,s)$ is decreasing in $\ell$, we may make reduce this case to the previous cases as follows.
\begin{align*}
\hat \rho^*(\ell,s) & \geq \hat \rho^*\left(\left(1-\frac{1}{8(s+1)^2}\right) \frac{sn}{s+1},s\right) \stackrel{\text{Case 1}}{\geq} \frac{1}{8(s+1)^2n} \stackrel{\ell \geq \kappa n}{\geq} \frac{\kappa}{8(s+1)^2\ell},
\end{align*}
which implies the first inequality. Analogously, using the slightly simpler bound $\kappa n \leq \ell \leq sn/(s+1)$, the second inequality follows from 
\begin{align*}
\hat \rho^*(\ell,s) & \leq \hat \rho^*(\kappa n,s) \leq \frac{\ln(s+1)}{\kappa n} \stackrel{\ell \leq sn/(s+1)}{\leq} \frac{s\ln(s+1)}{\kappa(s+1)\ell},
\end{align*}
which proves the second inequality. This concludes the proof.
\end{proof}

In the next lemma we give estimates for how much $\hat p_{\text{suc}}$ changes for mutation rates which slightly deviate from $\hat\rho^*$. This gives the analogue of Lemma~\ref{lem:successprob}.

\begin{lemma}[success probabilities around $\hat\rho^*$]\label{lem:successprob2}
There are constants $c= c(s)>0$ and $C=C(s)>0$, depending only on $s$ such that the following holds. Let $\hat\rho^*=\hat\rho^*(\ell,s)$ be the target mutation rate, i.e., the unique solution of \eqref{eq:rhostarhat}, and let $\eta_0 = \eta_0(n) = o(1)$ and $\ell_0 = (1-\eta_0)sn/(s+1)$ as in~\eqref{eq:ell0}. Then for all sufficiently large $n$, the following holds.
  \begin{enumerate}
  \item For all $0<\ell \leq \ell_0$ and all $\delta \in [0,c]$,
\begin{align*}
\hat p_{\text{suc}}((1+\delta)\hat\rho^*,\ell) & ~\leq \ \hat p_{\text{suc}}(\hat\rho^*,\ell) \cdot(1\phantom{{} - s\delta } - \tfrac{1}{2}\delta \hat\rho^*\ell + C\cdot (\delta^2 + \hat\rho^{*2}+ \delta\hat\rho^{*})), \\
\hat p_{\text{suc}}((1+\delta)\hat\rho^*,\ell) & ~\geq \ \hat p_{\text{suc}}(\hat\rho^*,\ell) \cdot(1-s\delta+\tfrac{s}{2}\delta \hat\rho^*\ell - C\cdot (\delta^2 + \hat\rho^{*2}+ \delta\hat\rho^{*})), \\
\hat p_{\text{suc}}((1-\delta)\hat\rho^*,\ell) & ~\geq \ \hat p_{\text{suc}}(\hat\rho^*,\ell) \cdot(1\phantom{{}+s\delta }+\tfrac{1}{2}\delta \hat\rho^*\ell - C\cdot (\delta^2 + \hat\rho^{*2}+ \delta\hat\rho^{*})) \\
\hat p_{\text{suc}}((1-\delta)\hat\rho^*,\ell) & ~\leq \ \hat p_{\text{suc}}(\hat\rho^*,\ell) \cdot(1+s\delta-\tfrac{s}{2}\delta \hat\rho^*\ell + C\cdot (\delta^2 + \hat\rho^{*2}+ \delta\hat\rho^{*})).
\end{align*}
  \item 
 For all $0\leq \ell \leq n$ and all $\rho \in [\ln(2s+2)/\ell,1]$.
\begin{align*}
\hat p_{\text{suc}}(\rho,\ell) & \leq  \frac12 \cdot \frac{1}{(s+1)}.
\end{align*}
  \item Let $\rho_{0}=\rho_{0}(\ell):= \tfrac14 \ln((2s+2)/(s+3)) /\ell$, and let $\zeta = \zeta(n) = o(1)$ be any falling function. If $n$ is sufficiently large, then the following holds for all $1\leq \ell \leq \zeta n$ and all $\rho \in(0,\rho_{0}]$. 
\begin{align*}
\hat p_{\text{suc}}(\rho,\ell) & \geq \frac{s+3}{2s+2} = \frac{1}{s+1}\left(1+\frac{s+1}{2}\right).
\end{align*}
  \item For all $\ell_0 < \ell \leq n$ and all $ \rho \geq \rho_0 := 4(s+1)\eta_0/n$,
\begin{align*}
\hat p_{\text{suc}}(\rho,\ell) & \leq  \frac{1}{s+1}\left(1-\frac{\eta_0 s}{10}\right).
\end{align*}
  \end{enumerate}
\end{lemma}

Before we prove Lemma~\ref{lem:successprob2}, let us briefly comment on the different cases. Part~(a)--(c) are concerned with the case $\ell < \ell_0$. By Lemma~\ref{lem:estimateshatrhostar} we have $\hat\rho^* = O(1/\ell)$, and even $\hat\rho^* = \Theta(1/\ell)$ if $\ell$ is bounded away from $\ell_0$. Thus part~(a) gives bounds that differ only by a factor $(1+o(1))$ if $\delta = o(1)$ and $\ell= \omega(1)$, where the latter condition is needed to bound the error term $\hat\rho^{*2} = O(1/\ell^2)$. However, for $\ell = O(1)$ the bound may be very bad, since then $\hat\rho^{*2} = \Omega(1)$ leads to a too large error term. Hence, we give bounds in (b) and (c), which are less precise for most values of $\ell$, but which give concrete bounds for $\ell=O(1)$. Since these $\ell$ constitute a very small fraction of all values of $\ell$, we can afford to work with less tight bounds in this case. Finally, case (d) deals with the case $\ell > \ell_0$, in which for all $\rho \geq \rho_0 = \Theta(\eta_0/n)$ the success rate stays at least by a factor $(1-\Omega(\eta_0))$ below the target success rate of $1/(s+1)$. We will use this fact later to argue that the state $\rho \geq \rho_0$ is unstable, and the algorithm will quickly converge to smaller values of $\rho$.
\begin{proof}[Proof of Lemma~\ref{lem:successprob2}]
\emph{(a)}. The proof will rely on asymptotic expansions of numerator and denominator of $\hat p_{\text{suc}}(\rho,\ell) = (1-(1-\rho)^\ell)/(1-(1-\rho)^n)$. As we will see, increasing (or decreasing) $\rho$ by a factor of $(1+\delta)$ will increase (decrease) both numerator and denominator by a factor of roughly $(1+\delta)$. However, the two factors differ in the second order error term, and this second order term will dominate the change of the quotient. For this reason we need to make the asymptotic expansions rather precise, up to third order error terms. Throughout the proof, all hidden constants in the $O$-notation are absolute constants that depend only on $s$.

The asymptotic expansions will rely on Lemma~\ref{lem:asym}. In the following calculations, the shorthand notation \ref{it:asym1},~\ref{it:asym2},~\ref{it:asym3} will refer to the corresponding parts of Lemma~\ref{lem:asym}. For convenience, we repeat part~\ref{it:asym4} of the lemma, since it is less standard than the expansions of $e^{-x}$, $\ln(1-x)$, and $1/(1-x)$. 
\begin{align}\label{eq:successprob2a}
1-e^{-y}(1+x) = (1-e^{-y})\left(1-x/y+x/2\pm O(y x)\right).
\end{align}

Recall that by Lemma~\ref{lem:estimateshatrhostar} we either have $\hat\rho^* =  \Theta(1/\ell)$ if $\ell$ is bounded away from $sn/(s+1)$ by a constant factor, or we have $\hat\rho^* =  \Theta(\eta/\ell) = \Theta(\eta/n)$ if $\ell = (1-\eta)sn/(s+1)$, where $\eta \geq \eta_0$. In both cases, we have $\hat\rho^*\ell = O(1)$. Note that this allows us to remove factors $\hat\rho^*\ell$ from error terms, as in $O(\hat\rho^{*3}\ell) \subseteq O(\hat\rho^{*2})$, which we will use to simplify and unify error terms in the upcoming calculations. Moreover, the assumptions on $\ell$ imply $\hat\rho^* = o(1)$. 
After these preparations, we can estimate the term $1-(1-\hat \rho^*)^\ell$ as follows.
\begin{align*}
1-(1-\hat\rho^*)^{\ell} & \stackrel{\ref{it:asym3}}{=} 1-\big(e^{-\hat\rho^*+\tfrac{1}{2}\hat\rho^{*2}\pm O(\hat\rho^{*3})}\big)^{\ell} \\
& = 1-e^{-\hat\rho^*\ell}\cdot e^{\tfrac{1}{2}\hat\rho^{*2}\ell \pm O(\hat\rho^{*3}\ell)} \\
& \stackrel{\ref{it:asym1}}{=} 1-e^{-\hat\rho^*\ell}\big(1+\tfrac{1}{2}\hat\rho^{*2}\ell \pm O(\hat\rho^{*3}\ell)\big)\\
& \stackrel{\eqref{eq:successprob2a}}{=} (1-e^{-\hat\rho^*\ell})(1-\tfrac12 \hat\rho^* + \tfrac{1}{4}\hat\rho^{*2}\ell \pm O(\hat\rho^{*2})).
\end{align*}
We can also turn the approximation around by dividing both sides through the second bracket on the right-hand side and using the expansion of $1/(1-x)$:
\begin{align}\label{eq:successprob2c}
1-e^{-\hat\rho^*\ell} \stackrel{\ref{it:asym2}}{=}  (1-(1-\hat\rho^*)^{\ell})(1+\tfrac12 \hat\rho^* - \tfrac{1}{4}\hat\rho^{*2}\ell \pm O(\hat\rho^{*2})).
\end{align}
Now we do an analogous computation for $1-(1-(1+\delta)\hat \rho^*)^\ell$. Since this part of the calculation goes through for positive and negative deviations alike, let us momentarily consider any $\delta \in [-c,c]$. Then
\begin{align*}
1-(1-(1+\delta)\hat\rho^*)^{\ell} & \stackrel{\ref{it:asym3}}{=} 1-\big(e^{-(1+\delta)\hat\rho^*+\tfrac{1}{2}(1+\delta)^2\hat\rho^{*2}-O(\hat\rho^{*3})}\big)^{\ell} \\
& = 1-e^{-\hat\rho^*\ell}\cdot e^{\tfrac{1}{2}\hat\rho^{*2}\ell-\delta\hat\rho^{*}\ell + \delta \hat\rho^{*2}\ell \pm O(\hat\rho^{*3}\ell+\delta^2\hat\rho^{*2}\ell)} \\
& \stackrel{\ref{it:asym1}}{=} 1-e^{-\hat\rho^*\ell}\big(1+\tfrac{1}{2}\hat\rho^{*2}\ell-\delta\hat\rho^{*}\ell + \delta \hat\rho^{*2}\ell  \pm O(\hat\rho^{*3}\ell+\delta^2\hat\rho^{*2}\ell^2+\delta\hat\rho^{*3}\ell^2)\big)\\
& \stackrel{\eqref{eq:successprob2a}}{=} (1-e^{-\hat\rho^*\ell})(1-\tfrac12 \hat\rho^* +\delta  + \tfrac{1}{4}\hat\rho^{*2}\ell - \tfrac12\delta\hat\rho^{*}\ell\pm O(\hat\rho^{*2}  +\delta^2\hat\rho^{*}\ell+ \delta\hat\rho^{*})).
\end{align*}
We plug this equation into~\eqref{eq:successprob2c}, divide both sides by $1-(1-\hat\rho^*)^{\ell}$ and multiply out the right hand side, and obtain for all $\delta \in [-c,c]$,
\begin{align}\label{eq:successprob2e}
\frac{1-(1-(1+\delta)\hat\rho^*)^{\ell}}{1-(1-\hat\rho^*)^{\ell}} 
& = 1 +\delta  - \tfrac12\delta\hat\rho^{*}\ell\pm O(\hat\rho^{*2}+\delta^2\hat\rho^{*}\ell+\delta\hat\rho^{*}).
\end{align}
Unfortunately, we cannot easily replicate the calculation for $n$, since the bound $\hat\rho^{*}\ell = O(1)$ does not have an analogue for $n$. However, we can replicate the calculation for $\alpha\ell$ instead of $\ell$, where we set $\alpha := \frac{s+1}{s}$ for concreteness. Note that we have chosen $\alpha$ such that $n > \alpha \ell$. Now consider the function
\begin{align*}
f(\delta,x) := \frac{1-(1-(1+\delta)\hat\rho^*)^{x}}{1-(1-\hat\rho^*)^x}.
\end{align*}
The same argument as above gives $f(\delta,\alpha\ell) = 1 +\delta  - \tfrac12\delta\hat\rho^{*}\alpha\ell\pm O(\hat\rho^{*2}+\delta^2\hat\rho^{*}\ell+\delta\hat\rho^{*})\big).$
So far we have considered arbitrary $\delta$, but let us now constrict to $\delta \in (0,c]$. Then the function $f$ is of the form $(1-b^x)/(1-c^x)$ with $b < c$, which is strictly decreasing for $x\in \R^+$ by Lemma~\ref{lem:auxmonotone}. Hence, for $\delta \in (0,c]$,
\begin{align}\label{eq:successprob2g}
1<\frac{1-(1-(1+\delta)\hat\rho^*)^{n}}{1-(1-\hat\rho^*)^{n}} = f(\delta,n) < f(\delta,\alpha\ell) =  1 +\delta  - \tfrac12\delta\hat\rho^{*}\alpha\ell\pm O(\hat\rho^{*2}+\delta^2\hat\rho^{*}\ell+\delta\hat\rho^{*}).
\end{align}
Combining this with~\eqref{eq:successprob2e}, we obtain for $\delta \in [0,c]$,
\begin{align*}
\frac{1-(1-(1+\delta)\hat\rho^*)^{\ell}}{1-(1-(1+\delta)\hat\rho^*)^{n}}\Big/\frac{1-(1-\hat\rho^*)^{\ell}}{1-(1-\hat\rho^*)^{n}} & \geq \frac{1 +\delta  - \tfrac12\delta\hat\rho^{*}\ell\pm O(\hat\rho^{*2}+\delta^2\hat\rho^{*}\ell+\delta\hat\rho^{*})}{1 +\delta  - \tfrac12\delta\hat\rho^{*}\alpha\ell\pm O(\hat\rho^{*2}+\delta^2\hat\rho^{*}\ell+\delta\hat\rho^{*})} \\
& \stackrel{\ref{it:asym2}}{=} 1+\tfrac12\delta\hat\rho^{*}(\alpha-1)\ell \pm O(\hat\rho^{*2}+\delta^2+\delta\hat\rho^{*}) =: 1+\beta.
\end{align*}
Note that the left hand side equals $(1-\hat p_{\text{suc}}((1+\delta)\hat\rho^*,\ell))/(1-\hat p_{\text{suc}}(\hat\rho^*,\ell))$. Since $\hat p_{\text{suc}}(\hat\rho^*,\ell) = 1/(s+1)$, we thus have $(1-\hat p_{\text{suc}}((1+\delta)\hat\rho^*,\ell)) \geq (1+\beta)s/(s+1)$, or equivalently
\begin{align*}
\hat p_{\text{suc}}((1+\delta)\hat\rho^*,\ell) \leq 1-\frac{s}{s+1}(1+\beta) = \frac{1}{s+1}(1-s \beta) = \hat p_{\text{suc}}(\hat\rho^*,\ell)\cdot (1-s\beta).
\end{align*}
Using $\alpha-1 = 1/s$, this proves the first inequality. The second follows by the same calculation, but using the trivial ``$>1$'' bound from~\eqref{eq:successprob2g}. For the other two inequalities, let again $\delta \in [0,c]$ and observe that the function $f(-\delta,y)$ is of the form $\left((1-b^y)/(1-c^y)\right)^{-1}$ with $b < c$, and thus it is strictly increasing in $y$. Hence, for $\delta \in(0,c]$,
\begin{align}\label{eq:successprob2i}
\frac{1-(1-(1-\delta)\hat\rho^*)^{n}}{1-(1-\hat\rho^*)^{n}} = f(-\delta,n) > f(-\delta,\alpha\ell) =  1 -\delta + \tfrac12\delta\hat\rho^{*}\alpha\ell\pm O(\hat\rho^{*2}+\delta^2\hat\rho^{*}\ell+\delta\hat\rho^{*}),
\end{align}
which yields by plugging in $-\delta$ into~\eqref{eq:successprob2e},
\begin{align*}
\frac{1-(1-(1-\delta)\hat\rho^*)^{\ell}}{1-(1-(1-\delta)\hat\rho^*)^{n}}\Big/\frac{1-(1-\hat\rho^*)^{\ell}}{1-(1-\hat\rho^*)^{n}} & \leq \frac{1 -\delta  + \tfrac12\delta\hat\rho^{*}\ell\pm O(\hat\rho^{*2}+\delta^2\hat\rho^{*}\ell+\delta\hat\rho^{*})}{1 -\delta  + \tfrac12\delta\hat\rho^{*}\alpha\ell\pm O(\hat\rho^{*2}+\delta^2\hat\rho^{*}\ell+\delta\hat\rho^{*})} \\
& = 1-\tfrac12\delta\hat\rho^{*}(\alpha-1)\ell \pm O(\hat\rho^{*2}+\delta^2+\delta\hat\rho^{*}) =: 1-\beta'.
\end{align*}
As before, we may deduce $\hat p_{\text{suc}}((1+\delta)\hat\rho^*,\ell) \geq \hat p_{\text{suc}}(\hat\rho^*,\ell)\cdot (1+s\beta')$, which proves the third inequality. For the fourth, we repeat the same calculation, but use that the left hand side of~\eqref{eq:successprob2i} is trivially at most $1$. This concludes the proof of (a). 

For later reference, we claim that the same derivation as above also shows that for any constant $c >0$ we may write for all $\delta \in [-c,c]$ that
\begin{align}\label{eq:successprob2k}
\frac{1-(1-(1+\delta)\hat\rho^*)^{n-\ell}}{1-(1-\rho)^{n-\ell}} =  1 \pm O(|\delta| +\hat\rho^{*2}),
\end{align}
where the hidden constants may depend on $c$ and $s$. Indeed, if $n-\ell \geq \alpha \ell$, then we may replace $n$ by $n-\ell$ in~\eqref{eq:successprob2i}, which is a stronger statement than~\eqref{eq:successprob2k}. Otherwise we have $\ell \geq (\alpha+1)n$, and hence $n-\ell = \Theta(n) \subseteq \Theta(\ell)$, and thus the derivation of ~\eqref{eq:successprob2e} remains valid if we replace $\ell$ by $n-\ell$. This proves~\eqref{eq:successprob2k}.

\noindent\emph{(b)}. By Lemma~\ref{lem:rhostarhat} the function $\hat p_{\text{suc}}(\rho,\ell)$ is strictly decreasing in $\rho$. Thus we may estimate
\begin{align*}
\hat p_{\text{suc}}\left(\rho,\ell\right) & \leq \hat p_{\text{suc}}\left(\frac{\ln(2s+2)}{\ell},\ell\right) \leq 1-\frac{1-(1-\frac{\ln(2s+2)}{\ell})^{\ell}}{1} = \Big(1-\frac{\ln(2s+2)}{\ell}\Big)^{\ell} \\
& \leq e^{-\ln(2s+2)} = 1/(2s+2).
\end{align*}

\noindent\emph{(c)}. Again, by Lemma~\ref{lem:rhostarhat} the function $\hat p_{\text{suc}}(\rho,\ell)$ is strictly decreasing in $\rho$. Hence
\begin{align*}
\hat p_{\text{suc}}\left(\rho,\ell\right) & \geq \hat p_{\text{suc}}\left(\rho_{0},\ell\right) = 1-\frac{1-(1-\rho_{0})^{\ell}}{1-(1-\rho_{0})^{n}} \geq 1-\frac{1-e^{-2\rho_{0} \ell}}{1-o(1)} \\
& \stackrel{\text{$n$ large}}{\geq} 1-\Big(1-\frac{s+3}{2s+2}\Big) = \frac{s+3}{2s+2}
\end{align*}

\noindent\emph{(d)}. By Lemma~\ref{lem:rhostarhat}, the function $\hat p_{\text{suc}}(\rho,\ell)$ is strictly decreasing in $\rho$, and by~\eqref{eq:qhat} it is strictly decreasing in $\ell$. Hence, 
\[
\hat p_{\text{suc}}(\rho,\ell) \leq \hat p_{\text{suc}}(\rho_0,\ell_0) \stackrel{\eqref{eq:estimateshatrhostar1a}}{\leq} \frac{1}{s+1}\left(1-\frac{\eta_0 s}{10}\right).
\]
This concludes the proof.
\end{proof}

As final preparation for the main proof, we estimate the improvement probability for mutation rates that (as we will show) the algorithm is mostly using.

\begin{lemma}\label{lem:strictimp2}
There is a constant $C=C(s)>0$ such that the following holds. Let $\eta_0 = o(1), \gamma = o(1)$ and let $\ell_0=(1-\eta_0)sn/(s+1)$. 
\begin{enumerate}
\item If $n$ is sufficiently large, then for all $\ell \in [1..\ell_0]$ the following holds. Let $\hat\rho^* := \hat\rho^*(\ell,s)$, and let $\rho \in [(1-\gamma) \hat \rho^*,(1+\gamma) \hat \rho^*]$. Then 
\[
\hat p_{\text{imp}}(\hat\rho^*,\ell) (1 - C(\gamma+1/n)) \leq \hat p_{\text{imp}}(\rho,\ell) \le \hat p_{\text{imp}}(\hat\rho^*,\ell) (1 + C(\gamma+1/n)).
\]
In other words, $\hat p_{\text{imp}}(\rho,\ell) = (1\pm O(\gamma + 1/n))\hat p_{\text{imp}}(\hat\rho^*,\ell)$, uniformly over all $\ell \in [1..\ell_0]$ and all $\rho \in [(1-\gamma) \hat \rho^*,(1+\gamma) \hat \rho^*]$.
\item For all $\rho \leq \eta_0/n$ and all $\ell \geq \ell_0$, 
\[
\frac{1-C \eta_0}{n} \leq \hat p_{\text{imp}}(\rho,\ell) \leq \frac{1+C \eta_0}{n}
\]
In particular, $\hat p_{\text{imp}}(\rho,\ell) = (1\pm o(1))/n$.
\end{enumerate}
\end{lemma}

\begin{proof}
\emph{(a)}. Consider first the case $\ell \geq \sqrt{n}$, where the value $\sqrt{n}$ is chosen rather arbitrarily. We use the formula $\hat p_{\text{imp}}(\rho,\ell) = \hat p_{\text{suc}}(\rho,\ell)\cdot \rho/(1-(1-\rho)^{n-\ell})$, and bound all three factors independently. For the factor $\rho$, it is trivial that $\rho = (1\pm O(\gamma))\hat\rho^*$. For the other two factors, by Lemma~\ref{lem:estimateshatrhostar} we have $\hat\rho^* \ell = O(1) \cap \Omega(\eta_0)$. Thus, by Lemma~\ref{lem:successprob2} (a),
\[
\hat p_{\text{suc}}(\rho,\ell)  = \hat p_{\text{suc}}(\hat\rho^*,\ell)(1 \pm O(\gamma \hat\rho^*\ell + \gamma^2 + \hat\rho^{*2}+ \gamma\hat\rho^{*})) = \hat p_{\text{suc}}(\hat\rho^*,\ell)(1 \pm O(\gamma + 1/n)).
\]
For the third term, since $n-\ell = \Omega(n)$, we have shown in~\eqref{eq:successprob2k} that 
\[
\frac{1-(1-\rho)^{n-\ell}}{1-(1-\hat\rho^*)^{n-\ell}} =  1 \pm O(\gamma+\hat\rho^{*2}) = 1 \pm O(\gamma+1/n).
\]
Altogether, all three terms give factors of the form $1 \pm O(\gamma+1/n)$ if we vary $\rho$, and so $\hat p_{\text{imp}}(\rho,\ell)$ deviates from $\hat p_{\text{imp}}(\hat\rho^*,\ell)$ by a factor of the same form.

For $\ell < \sqrt{n}$, we directly use the formula $\hat p_{\text{imp}}(\rho,\ell) = (1-\rho)^{\ell}\rho/(1-(1-\rho)^n)$. In this regime, since $\hat\rho^* = \Theta(1/\ell)$ by Lemma~\ref{lem:estimateshatrhostar}, we may use trivial bounds on the denominator:
\[
1 \geq 1- (1-\rho)^{n} = 1-e^{-\Omega(\rho n)} = 1- e^{-\Omega(\sqrt{n})} = 1-O(1/n).
\]
These trivial bounds show that any two expressions of the form $1- (1-\rho)^{n}$ can deviate at most by a factor $1\pm O(1/n)$. For the other two factors of $\hat p_{\text{imp}}(\rho,\ell)$, the factor $\rho$ again satisfies trivially $\rho = (1\pm O(\gamma))\hat\rho^*$. Finally, for the factor $(1-\rho)^\ell$ we use the estimate
\[
\frac{(1-\rho)^\ell}{(1-\hat\rho^*)^{\ell}} = \left(1-\frac{\rho-\hat\rho^*}{1-\hat\rho^*}\right)^{\ell} = (1\pm O(\gamma\hat\rho^*))^{\ell} = e^{\pm O(\gamma\hat\rho^*\ell)} = e^{\pm O(\gamma)} = 1 \pm O(\gamma).
\]
Again, all three factors of $\hat p_{\text{imp}}(\rho,\ell)$ deviate at most by factors $1 \pm O(\gamma+1/n)$ if we vary $\rho$, and thus $\hat p_{\text{imp}}(\rho,\ell) = \hat p_{\text{imp}}(\hat\rho^*,\ell)(1\pm O(\gamma+1/n))$.

\emph{(b)}. We will evaluate the formula $\hat p_{\text{imp}}(\rho,\ell) = (1-\rho)^{\ell}\rho/(1-(1-\rho)^n)$. For the denominator, we again use the asymptotic expansion
\begin{align*}
1-(1-\rho)^{n} &= 1-e^{-(\rho + O(\rho^2))n}=1-e^{-\rho n(1-O(\eta_0))} = 1-(1-\rho n(1 \pm O(\eta_0))) \\
& = \rho n (1\pm O(\eta_0)). 
\end{align*}
Similarly, we also get $(1-\rho)^{\ell} = 1-O(\rho\ell) = 1-O(\eta_0)$. Hence, 
\[
\hat p_{\text{imp}}(\rho,\ell) = \frac{(1-\rho)^{\ell}\rho}{(1-(1-\rho)^n)} =  \frac{(1- O(\eta_0))\rho}{\rho n (1 \pm O(\eta_0))} = (1\pm O(\eta_0)) \frac{1}{n}.
\]
\end{proof}

With Lemmas~\ref{lem:successprob2} and~\ref{lem:strictimp2} at hand, the analysis for $0 \leq \ell \leq \ell_0$ is almost completely analogous to the case of the \oea, and we only give a sketch. 

\begin{lemma}\label{lem:gaintime2}
Let $c>1$ be a constant. Consider a run of the self-adjusting \oeares with update strength $F=1+\eps$ for some $\eps = \omega(\frac{\log n}{n}) \cap o(1)$ and with $\pmin \in o(n^{-1}) \cap \Omega(n^{-c})$, $\pmax = 1$, on the $n$-dimensional \leadingones function. Let $\eta_0 = \max\{\eps^{1/6},(n\eps/\log n)^{-1/2}\}$ and let $\ell_0 = (1-\eta_0)sn/(s+1)$.  Assume that the self-adjusting \oeares is started with a search point of fitness $\ell \in [0..\ell_0]$ and an arbitrary mutation rate $\rho \in [\pmin,\pmax]$. Let $\hat \rho^* = \hat \rho^*(\ell,s)$. Then the number $T_\ell$ of iterations until a search point with fitness better than $\ell$ is found is stochastically dominated by
  \[T_\ell \preceq o(n) + \Geom(\min\{\omega(\tfrac 1n), (1-o(1))(1-\hat\rho^*)^\ell \hat\rho^*/(1-\hat\rho^*)^n\}).\]
  In particular, $E[T_\ell] \le o(n) + \frac{(1-\hat\rho^*)^n}{(1-\hat\rho^*)^\ell \hat\rho^*}$. All hidden factors in the asymptotic notation can be chosen independently of $\ell$.
\end{lemma}

\begin{proof}[Proof Sketch.]
We only outline the parts that differ from the analysis of the \oea. Consider first the case that $\zeta n \leq \ell \leq \ell_0$, where we choose $\zeta := \eps^{1/3}$. Then by Lemma~\ref{lem:estimateshatrhostar} we have $\hat\rho^* = O(1/(\zeta n)) \cap \Omega(\eta_0/n)$ and $\hat\rho^* \ell = O(1) \cap \Omega(\eta_0)$. Moreover, we also choose $\delta = \eps^{1/3}$. By Lemma~\ref{lem:successprob2} (a), 
\begin{align*}
\hat p_{\text{suc}}((1+\delta)\hat\rho^*,\ell) &\le \hat p_{\text{suc}}(\hat\rho^*,\ell) \cdot(1-\tfrac{1}{2}\delta \hat\rho^*\ell \pm O(\delta^2 + \hat\rho^{*2}+ \delta\hat\rho^{*})), \text{ and}\\
\hat p_{\text{suc}}((1-\delta)\hat\rho^*,\ell) &\ge \hat p_{\text{suc}}(\hat\rho^*,\ell) \cdot(1+\tfrac{1}{2}\delta \hat\rho^*\ell \pm O(\delta^2 + \hat\rho^{*2}+ \delta\hat\rho^{*})).
\end{align*}
By our choices, $\delta^2 = \delta \eps^{1/3} \subseteq o(\delta \hat\rho^*\ell)$ and $\hat\rho^{*2} = O(\hat\rho^{*}/(\zeta n)) \subseteq o(\delta \hat\rho^*\ell)$. Therefore, the error terms are of minor order, and we may bound
\begin{align*}
\hat p_{\text{suc}}((1+\delta)\hat\rho^*,\ell) &\le \hat p_{\text{suc}}(\hat\rho^*,\ell) \cdot(1-(1\pm o(1))\tfrac{1}{2}\delta \hat\rho^*\ell), \text{ and }\\
\hat p_{\text{suc}}((1-\delta)\hat\rho^*,\ell) &\ge \hat p_{\text{suc}}(\hat\rho^*,\ell) \cdot(1+(1\pm o(1))\tfrac{1}{2}\delta \hat\rho^*\ell).
\end{align*}
Note that, up to minor order terms, this is exactly the same expression as for the \oea in Lemma~\ref{lem:successprob}, since we consider $\delta < 1/\ln(s+1)$.  Thus the proof of Lemma~\ref{lem:rhoisgood} carries over, and after a short initial phase most rounds will be spent with mutation rates $\rho \in [(1-\delta)\hat\rho^*,(1+\delta)\hat\rho^*]$. Moreover, by Lemma~\ref{lem:strictimp2} (a), for any $\rho \in [(1-\delta)\hat\rho^*, (1+\delta)\hat\rho^*]$ the improvement probability is $\hat p_{\text{imp}}(\rho,\ell) = (1 \pm o(1))\hat p_{\text{imp}}(\hat\rho^*,\ell)$, where $\hat p_{\text{imp}}(\hat\rho^*,\ell) = (1-\hat\rho^*)^\ell \hat\rho^*/(1-\hat\rho^*)^n$. Thus the proof of Lemma~\ref{lem:gaintime} also carries over, and we obtained the domination statement as claimed. Note that with our choices we have $\delta/\log(1/\delta) = \omega(\eps)$ and $\delta = \omega(\frac{\log n}{n\eps})$, as required in the proof of Lemma~\ref{lem:gaintime}.

For $0 <\ell < \zeta n$ we use the same proof strategy, but don't need to care about constant factors anymore, since the expected time to leave the level is $o(n)$. Moreover, we know that $\hat\rho^* = \Theta(1/\ell)$ in this regime by Lemma~\ref{lem:estimateshatrhostar}. By Lemma~\ref{lem:successprob2} (b) and (c), there are constants $c_1,c_2 >0$ such that the mutation rate $\rho$ tends to the interval $I := [c_1/\ell,c_2/\ell]$ in the following sense. Whenever the mutation rate $\rho$ is larger than $c_2/\ell$ then the success rate $\hat p_{\text{suc}}(\rho, \ell)$ is larger than $1/(s+1)$ by a constant factor. As in the proof of Lemma~\ref{lem:approachrho}, this implies that $\rho$ has a multiplicative drift with factor $1-\Omega(\eps)$. Similarly, if the mutation rate is smaller than $c_1/\ell$ then the success rate $\hat p_{\text{suc}}(\rho, \ell)$ is smaller than $1/(s+1)$ by a constant factor, and $1/\rho$ has a multiplicative drift with factor $1-\Omega(\eps)$. Consequently, with probability $1-o(1)$ the mutation rate will reach the interval $I$ within $O(\log n/\eps)$ rounds, or a new level is reached. As in the proof of Lemma~\ref{lem:gaintime} we can argue that in each subsequent round, $\rho$ is in the interval $I$ with probability $1-o(1)$, and we may translate this property into a domination statement. The only difference is that in the interval $I$ we know $\rho$ only up to a constant factor, so we also know the improvement probability $\hat p_{\text{imp}}(\rho,\ell)$ only up to a constant factor. Hence, we obtain that $T_{\ell}$ is dominated by $o(n) + \Geom(\min\{\omega(1/n), O(\hat p_{\text{imp}}(\hat\rho^*,\ell))\}$. Since $\hat p_{\text{imp}}(\hat\rho^*,\ell) = \Theta(1/\ell) = \omega(1/n)$, we get $T_\ell \preceq o(n) + \Geom(\omega(1/n))$, and the claim follows.
\end{proof}

The following lemma gives the corresponding statement for the case $\ell \geq \ell_0$.

\begin{lemma}\label{lem:gaintime3}
Let $c>1$ be a constant. Consider a run of the self-adjusting \oeares with update strength $F=1+\eps$ for some $\eps = \omega(\frac{\log n}{n}) \cap o(1)$ and with $\pmin \in o(n^{-1}) \cap \Omega(n^{-c})$, $\pmax = 1$, on the $n$-dimensional \leadingones function. Let $\eta_0 = \max\{\eps^{1/6},(n\eps/\log n)^{-1/2}\}$ and let $\ell_0 = (1-\eta_0)sn/(s+1)$. Assume that the self-adjusting \oeares is started with a search point of fitness $\ell \in [\ell_0..n]$ and an arbitrary mutation rate $\rho \in [\pmin,\pmax]$. Then the number $T_\ell$ of iterations until a search point with fitness better than $\ell$ is found is stochastically dominated by
  \[T_\ell \preceq o(n) + \Geom((1-o(1))/n).\]
  In particular, $E[T_\ell] \le (1+o(1))n$. The hidden constants in the $o$-notation can be chosen independently of $\ell$. 
\end{lemma}

\begin{proof}[Proof Sketch]
Again the proof is analogous to the proof of Lemma~\ref{lem:gaintime}, and we only outline the differences. By Lemma~\ref{lem:successprob2}, whenever the mutation rate $\rho$ satisfies $\rho \geq \rho_0 := \eta_0/n$ then the success probability is at most $(1-\eta_0 s/10)/(s+1) = (1-\Omega(\eta_0)) \cdot 1/(s+1)$. As in the proof of Lemma~\ref{lem:approachrho}, we conclude that $1/\rho$ has a multiplicative drift towards $1/\rho_0$ with drift factor $1-\Omega(\eps \eta_0) = 1-\omega(\log n/n)$. Thus with probability $1-o(1)$ the mutation rate falls below $\rho_0$ within $O(\log n/(\eps \eta_0)) = o(n)$ rounds, or a new level is reached. As in the proof of Lemma~\ref{lem:gaintime}, in each subsequent round $\rho$ is in the interval $I$ with probability $1-o(1)$. By Lemma~\ref{lem:strictimp2} (b), in each such round the improvement probability is at least $(1-o(1))/n$. Hence, as before we obtain that $T_\ell \preceq o(n) + \Geom(\min\{\omega(1/n), (1-o(1))/n)\}$. Simplifying the minimum yields $T_\ell \preceq o(n) + \Geom((1-o(1))/n))$, as claimed.
\end{proof}

With the domination statements, our main result on the \oeares follows directly from Theorem~\ref{thm:level}.

\begin{proof}[Proof of Theorem~\ref{thm:oeares}]
Lemmas \ref{lem:gaintime2} and~\ref{lem:gaintime3} give a domination bound on $T_{\ell}$ for all $\ell \in [0..n]$. Theorem~\ref{thm:oeares} thus follows from Theorem~\ref{thm:level} in the same way as Theorem~\ref{thm:oea}. Note that the phrasing in Theorem~\ref{thm:level} covers both the \oea and the \oeares. We omit the details.
\end{proof}

\section{Fixed-Target Running Times} 
\label{sec:fixedTarget}

Our main focus in the previous sections, and in particular in the sections presenting numerical evaluations of the self-adjusting \oea variants (i.e., Sections~\ref{sec:numerics-oea} and~\ref{sec:numerics-oeares}), was on computing the expected optimization time. We now follow a suggestion previously made in~\cite{CarvalhoD17}, and study the anytime performance of the algorithms, by analyzing their expected \emph{fixed-target running times} (a notion introduced already in~\cite{DoerrJWZ13}). That is, for an algorithm $A$ we regard for each target value $0\le v\le n$ the expected number $E[T(n,A,v)]$ of function evaluations needed by algorithm $A$ until it evaluates for the first time a solution $x$ which satisfies $\LO(x)\ge v$. 

\begin{figure}[t]
\centering
\includegraphics[width=\linewidth]{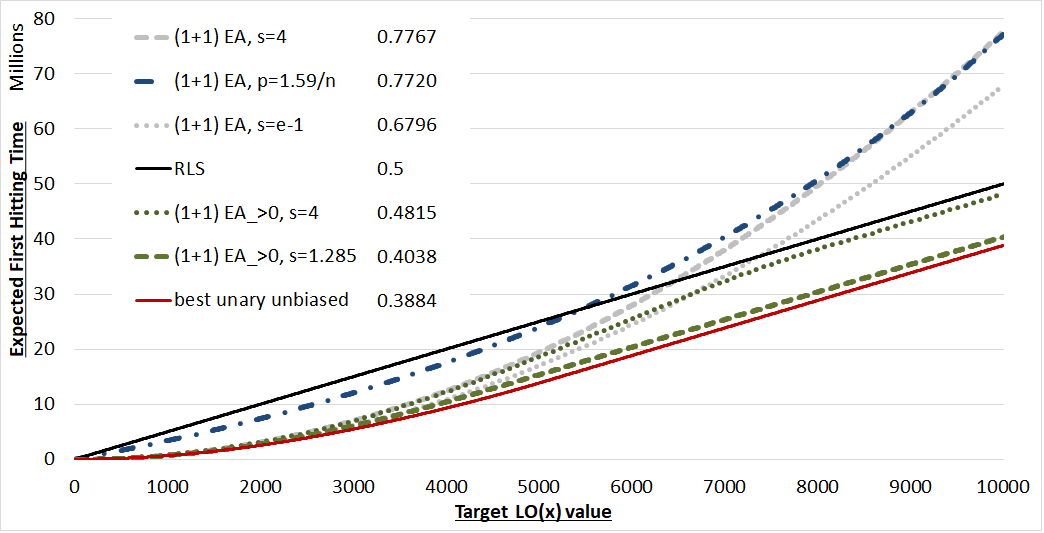}
\caption{Expected fixed target running times for \leadingones in dimension $n=10\,000$. The curve of \oearesopt is indistinguishable from that of the \oeares with success ratio $s=1.285$ and the curve of the \oeaopt indistinguishable from that of the \oea with success ratio $s=e-1$. The values shown in the legend are the expected optimization times, normalized by $1/n^2$.}
\label{fig:fixedTarget}
\end{figure} 

Figure~\ref{fig:fixedTarget} plots these expected fixed target running times of selected algorithms for $n=10\,000$. The legend also mentions the normalized expected overall optimization time, i.e., $E[T(n=10\,000,A,v)]/n^2$. We do not plot the \oearesopt, since its running time would be indistinguishable in this plot from the self-adjusting \oeares with success ratio $s=1.285$. For the same reason we do not plot the \oeaopt (i.e., the \oea with optimal fitness-dependent mutation rate $p= n/(\ell+1)$), whose data is almost identical to that of the self-adjusting \oea with the optimal success ratio $s=e-1$. 

We plot in Figure~\ref{fig:fixedTarget} the \oea with one-fifth success rule (i.e., success ratio $s=4$). While its overall running time is the worst of all algorithms plotted in this figure, we see that its fixed-target running time is better than that for RLS for all targets up to $6\,436$. Its overall running time is very close to that of the \oea with the best static mutation rate $p \approx 1.59/n$~\cite{BottcherDN10}, and for all targets $v \le 9\,017$ the expected running time is smaller.

We already discussed that the expected optimization time of the two algorithms \oeaopt and the self-adjusting \oea with success ratio $s=e-1$ is around $36\%$ worse than that of RLS. However, we also see that their fixed-target performances are better for all targets up to $v=7\,357$. For example, for $v=5\,000$ their expected first hitting time is slightly less than $17*10^6$ and thus about $36\%$ smaller than that of RLS.

As we have seen already in Figure~\ref{fig:oearesbyrate}, the self-adjusting \oeares with success ratio $s=4$ (i.e., using a one-fifth success rule) has an overall running time similar, but slightly better than RLS. We recall that its target mutation rate is 0 for values $v \ge 4n/5$. In this regime, the slope of its fixed target running-time curve is thus identical to that of RLS. For the self-adjusting \oeares with $s=1.285$ this is the case for $v$ slightly larger than $5\,600$. The \oearesopt with optimal fitness-dependent mutation rate uses mutation rate $p=0$ for $v \ge 4\,809$.

We also observe that the best unary unbiased black-box algorithm for with fitness-dependent mutation strength, which is the RLS-variant using in each step the $\mut_{k}$ operator with $k=\lfloor n /(\LO(x)+1) \rfloor$ (see~\cite{DoerrW18,Doerr19tcs} for more detailed discussions), is also best possible for all intermediate targets $v<n$. It is not difficult to verify this formally, the main argument being that the fitness landscape of \leadingones is non-deceptive.

\section{Conclusions} 

We have proven upper bounds for the expected time needed by the \oea and \oeares with success-based multiplicative update rules using constant success ratio $s$ and update strengths $F=1+o(1)$ to optimize the $n$-dimensional \leadingones function. In particular, we have shown that the \oea with one-$e$-th success rule achieves the running time of the \oea with optimal fitness-dependent mutation rate up to minor order terms (for update strengths $F=1+o(1)$). For the \oeares, numerical evaluations for $n=10\,000$ and $n=50\,000$ suggest a success ratio of around $1.285$, i.e., a one-$2.285$-th success rule. With this success rate the self-adjusting \oeares achieves an expected running time around $0.40375 n^2 +o(n^2)$, which, for $n=10\,000$, compares to a best possible expected running time (among all \oeares variants using fitness-dependent mutation rates) of $0.4027n^2 +o(n^2)$. Our precise upper bounds are stochastic domination bounds, which allow to derive other moments of the running time.  

Our work continues a series of recent papers rigorously demonstrating advantages of controlling the parameters of iterative heuristics during the optimization process. Developing a solid understanding of problems for which simple success-based update schemes are efficient, and which problems require more complex control mechanisms (e.g., based on reinforcement learning~\cite{DoerrDY16PPSN}, or techniques using statistics for the success rate within a window of iterations~\cite{LissovoiOW17,DoerrLOW18}) is the long-term goal of our research. 

\subsection*{Acknowledgments}

Our work was supported by a public grant as part of the
Investissement d'avenir project, reference ANR-11-LABX-0056-LMH, and by the Paris Ile-de-France Region.

}

\newcommand{\etalchar}[1]{$^{#1}$}


\begin{thebibliography}{KMH{\etalchar{+}}04}

\bibitem[AM16]{AletiM16}
Aldeida Aleti and Irene Moser.
\newblock A systematic literature review of adaptive parameter control methods
  for evolutionary algorithms.
\newblock {\em ACM Computing Surveys}, 49:56:1--56:35, 2016.

\bibitem[Aug09]{Auger09}
Anne Auger.
\newblock Benchmarking the (1+1) evolution strategy with one-fifth success rule
  on the {BBOB}-2009 function testbed.
\newblock In {\em Companion Material for Proc. of Genetic and Evolutionary
  Computation Conference (GECCO'09)}, pages 2447--2452. ACM, 2009.

\bibitem[BDDV20]{BuzdalovDDV20}
Maxim Buzdalov, Benjamin Doerr, Carola Doerr, and Dmitry Vinokurov.
\newblock Fixed-target runtime analysis.
\newblock In {\em Proc. of Genetic and Evolutionary Computation Conference
  (GECCO'20)}, pages 1295--1303. ACM, 2020.

\bibitem[BDN10]{BottcherDN10}
S\"untje B{\"o}ttcher, Benjamin Doerr, and Frank Neumann.
\newblock Optimal fixed and adaptive mutation rates for the {L}eading{O}nes
  problem.
\newblock In {\em Proc. of Parallel Problem Solving from Nature (PPSN'10)},
  volume 6238 of {\em Lecture Notes in Computer Science}, pages 1--10.
  Springer, 2010.

\bibitem[CD17]{CarvalhoD17}
Eduardo {Carvalho Pinto} and Carola Doerr.
\newblock Discussion of a more practice-aware runtime analysis for evolutionary
  algorithms.
\newblock In {\em Proc. of Artificial Evolution (EA'17)}, pages 298--305, 2017.
\newblock Extended version available online at
  \url{https://arxiv.org/abs/1812.00493}.

\bibitem[CD18]{CarvalhoD18}
Eduardo {Carvalho Pinto} and Carola Doerr.
\newblock A simple proof for the usefulness of crossover in black-box
  optimization.
\newblock In {\em Proc. of Parallel Problem Solving from Nature (PPSN'18)},
  volume 11102 of {\em Lecture Notes in Computer Science}, pages 29--41.
  Springer, 2018.

\bibitem[CFSS08]{DaCostaGECCO08}
Lu{\'{\i}}s~Da Costa, {\'{A}}lvaro Fialho, Marc Schoenauer, and Mich{\`{e}}le
  Sebag.
\newblock Adaptive operator selection with dynamic multi-armed bandits.
\newblock In {\em Proc. of Genetic and Evolutionary Computation Conference
  (GECCO'08)}, pages 913--920. ACM, 2008.

\bibitem[DD18]{DoerrD18ga}
Benjamin Doerr and Carola Doerr.
\newblock Optimal static and self-adjusting parameter choices for the
  $(1+(\lambda,\lambda))$ genetic algorithm.
\newblock {\em Algorithmica}, 80:1658--1709, 2018.

\bibitem[DD20]{DoerrD20chapter}
Benjamin Doerr and Carola Doerr.
\newblock Theory of parameter control for discrete black-box optimization:
  Provable performance gains through dynamic parameter choices.
\newblock In {\em Theory of Evolutionary Computation: Recent Developments in
  Discrete Optimization}, pages 271--321. Springer, 2020.
\newblock Also available at \url{https://arxiv.org/abs/1804.05650}.

\bibitem[DDL19]{DoerrDL19}
Benjamin Doerr, Carola Doerr, and Johannes Lengler.
\newblock Self-adjusting mutation rates with provably optimal success rules.
\newblock In {\em Proc. of Genetic and Evolutionary Computation Conference
  ({GECCO}'19)}, pages 1479--1487. {ACM}, 2019.

\bibitem[DDY16]{DoerrDY16PPSN}
Benjamin Doerr, Carola Doerr, and Jing Yang.
\newblock $k$-bit mutation with self-adjusting $k$ outperforms standard bit
  mutation.
\newblock In {\em Proc. of Parallel Problem Solving from Nature (PPSN'16)},
  volume 9921 of {\em Lecture Notes in Computer Science}, pages 824--834.
  Springer, 2016.

\bibitem[Dev72]{Devroye72}
Luc Devroye.
\newblock {\em The compound random search}.
\newblock Ph.D. dissertation, Purdue Univ., West Lafayette, {IN}, 1972.

\bibitem[DJW12]{DoerrJW12}
Benjamin Doerr, Daniel Johannsen, and Carola Winzen.
\newblock Multiplicative drift analysis.
\newblock {\em Algorithmica}, 64:673--697, 2012.

\bibitem[DJWZ13]{DoerrJWZ13}
Benjamin Doerr, Thomas Jansen, Carsten Witt, and Christine Zarges.
\newblock A method to derive fixed budget results from expected optimisation
  times.
\newblock In {\em Proc. of Genetic and Evolutionary Computation Conference
  (GECCO'13)}, pages 1581--1588. ACM, 2013.

\bibitem[DK21]{DoerrK21gecco}
Benjamin Doerr and Timo K\"otzing.
\newblock Lower bounds from fitness levels made easy.
\newblock In {\em Proc. of Genetic and Evolutionary Computation Conference
  (GECCO'21)}, pages 1142--1150. {ACM}, 2021.

\bibitem[DLOW18]{DoerrLOW18}
Benjamin Doerr, Andrei Lissovoi, Pietro~S. Oliveto, and John~Alasdair
  Warwicker.
\newblock On the runtime analysis of selection hyper-heuristics with adaptive
  learning periods.
\newblock In {\em Proc. of Genetic and Evolutionary Computation Conference
  (GECCO'18)}, pages 1015--1022. ACM, 2018.

\bibitem[Doe19]{Doerr19tcs}
Benjamin Doerr.
\newblock Analyzing randomized search heuristics via stochastic domination.
\newblock {\em Theoretical Computer Science}, 773:115--137, 2019.

\bibitem[Doe20]{Doerr20bookchapter}
Benjamin Doerr.
\newblock Probabilistic tools for the analysis of randomized optimization
  heuristics.
\newblock In Benjamin Doerr and Frank Neumann, editors, {\em Theory of
  Evolutionary Computation: Recent Developments in Discrete Optimization},
  pages 1--87. Springer, 2020.
\newblock Also available at \url{https://arxiv.org/abs/1801.06733}.

\bibitem[DW18]{DoerrW18}
Carola Doerr and Markus Wagner.
\newblock On the effectiveness of simple success-based parameter selection
  mechanisms for two classical discrete black-box optimization benchmark
  problems.
\newblock In {\em Proc. of Genetic and Evolutionary Computation Conference
  (GECCO'18)}, pages 943--950. {ACM}, 2018.

\bibitem[DWY21]{DoerrWY21}
Benjamin Doerr, Carsten Witt, and Jing Yang.
\newblock Runtime analysis for self-adaptive mutation rates.
\newblock {\em Algorithmica}, 83:1012--1053, 2021.

\bibitem[EHM99]{EibenHM99}
{\'{A}}goston~E. Eiben, Robert Hinterding, and Zbigniew Michalewicz.
\newblock Parameter control in evolutionary algorithms.
\newblock {\em IEEE Transactions on Evolutionary Computation}, 3:124--141,
  1999.

\bibitem[FCSS10]{FialhoCSS10}
{\'{A}}lvaro Fialho, Lu{\'{\i}}s~Da Costa, Marc Schoenauer, and Mich{\`{e}}le
  Sebag.
\newblock Analyzing bandit-based adaptive operator selection mechanisms.
\newblock {\em Annals of Mathematics and Artificial Intelligence}, 60:25--64,
  2010.

\bibitem[HGO95]{HansenGO95}
Nikolaus Hansen, Andreas Gawelczyk, and Andreas Ostermeier.
\newblock Sizing the population with respect to the local progress in
  (1,$\lambda$)-evolution strategies - a theoretical analysis.
\newblock In {\em Proc.\ of Congress on Evolutionary Computation (CEC'95)},
  pages 80--85. IEEE, 1995.

\bibitem[JDJW05]{JansenJW05}
Thomas Jansen, Kenneth~A. De~Jong, and Ingo Wegener.
\newblock On the choice of the offspring population size in evolutionary
  algorithms.
\newblock {\em Evolutionary Computation}, 13:413--440, 2005.

\bibitem[KEH14]{KarafotiasEH14}
Giorgos Karafotias, {\'{A}}goston~E. Eiben, and Mark Hoogendoorn.
\newblock Generic parameter control with reinforcement learning.
\newblock In {\em Proc. of Genetic and Evolutionary Computation Conference
  ({GECCO}'14)}, pages 1319--1326. {ACM}, 2014.

\bibitem[KHE15]{KarafotiasHE15}
Giorgos Karafotias, Mark Hoogendoorn, and {\'{A}}goston~E. Eiben.
\newblock Parameter control in evolutionary algorithms: Trends and challenges.
\newblock {\em IEEE Transactions on Evolutionary Computation}, 19:167--187,
  2015.

\bibitem[KMH{\etalchar{+}}04]{KernMHBOK04}
Stefan Kern, Sibylle~D. M{\"{u}}ller, Nikolaus Hansen, Dirk B{\"{u}}che, Jiri
  Ocenasek, and Petros Koumoutsakos.
\newblock Learning probability distributions in continuous evolutionary
  algorithms - a comparative review.
\newblock {\em Natural Computing}, 3:77--112, 2004.

\bibitem[LOW17]{LissovoiOW17}
Andrei Lissovoi, Pietro~S. Oliveto, and John~Alasdair Warwicker.
\newblock On the runtime analysis of generalised selection hyper-heuristics for
  pseudo-{B}oolean optimisation.
\newblock In {\em Proc. of Genetic and Evolutionary Computation Conference
  (GECCO'17)}, pages 849--856. ACM, 2017.

\bibitem[LOW20]{LissovoiOW19}
Andrei Lissovoi, Pietro~S. Oliveto, and John~Alasdair Warwicker.
\newblock Simple hyper-heuristics control the neighbourhood size of randomised
  local search optimally for {LeadingOnes}.
\newblock {\em Evolutionary Computation}, 28(3):437--461, 2020.

\bibitem[LS11]{LassigS11}
J{\"o}rg L{\"a}ssig and Dirk Sudholt.
\newblock Adaptive population models for offspring populations and parallel
  evolutionary algorithms.
\newblock In {\em Proc. of Foundations of Genetic Algorithms (FOGA'11)}, pages
  181--192. ACM, 2011.

\bibitem[LW12]{LehreW12}
Per~Kristian Lehre and Carsten Witt.
\newblock Black-box search by unbiased variation.
\newblock {\em Algorithmica}, 64:623--642, 2012.

\bibitem[RABD19]{RodionovaABD19}
Anna Rodionova, Kirill Antonov, Arina Buzdalova, and Carola Doerr.
\newblock Offspring population size matters when comparing evolutionary
  algorithms with self-adjusting mutation rates.
\newblock In {\em Proc. of Genetic and Evolutionary Computation Conference
  (GECCO'19)}, pages 855--863. ACM, 2019.

\bibitem[Rec73]{Rechenberg}
Ingo Rechenberg.
\newblock {\em Evolutionsstrategie}.
\newblock Friedrich Fromman Verlag (G{\"u}nther Holzboog {KG}), Stuttgart,
  1973.

\bibitem[SS68]{SchumerS68}
Michael~A. Schumer and Kenneth Steiglitz.
\newblock Adaptive step size random search.
\newblock {\em IEEE Transactions on Automatic Control}, 13:270--276, 1968.

\bibitem[Sud13]{Sudholt13}
Dirk Sudholt.
\newblock A new method for lower bounds on the running time of evolutionary
  algorithms.
\newblock {\em IEEE Transactions on Evolutionary Computation}, 17:418--435,
  2013.

\bibitem[Tre13]{trench2013introduction}
William~F. Trench.
\newblock {\em Introduction to real analysis}.
\newblock Open Textbook Initiative, American Institute of Mathematics, 2013.

\bibitem[Weg01]{Wegener01}
Ingo Wegener.
\newblock Theoretical aspects of evolutionary algorithms.
\newblock In Fernando Orejas, Paul~G. Spirakis, and Jan van Leeuwen, editors,
  {\em Proc. of the 28th International Colloquium on Automata, Languages and
  Programming (ICALP'01)}, volume 2076 of {\em Lecture Notes in Computer
  Science}, pages 64--78. Springer, 2001.

\end{thebibliography}
\end{document}